\documentclass{article}

\usepackage[accepted]{icml2024} 

\usepackage{graphicx}
\usepackage{epstopdf}
\usepackage{booktabs}
\usepackage{subcaption}
\usepackage{hyperref}
\usepackage{wrapfig}

\usepackage{multirow}
\usepackage{pythonhighlight}
\usepackage{pifont}

\usepackage{amsmath}
\usepackage{amssymb}
\usepackage{mathtools}
\usepackage{amsthm}
\usepackage{wrapfig}
\usepackage{algorithm}
\usepackage[noend]{algpseudocode}
\usepackage{arydshln}

\allowdisplaybreaks

\def\B{{\mathcal{B}}}
\def\R{{\mathbb{R}}}

\def\Ee{{\mathcal{E}}}

\def\E{\mathop{\mathbb{E}}}

\def\A{{\mathcal{A}}}
\def\AG{{\mathcal{AG}}}
\def\U{{\mathcal{U}}}
\def\UG{{\mathcal{UG}}}

\def\Dd{{\mathcal{D}}}
\def\L{{\mathcal{L}}}

\def\CL{\hbox{\rm\tiny CL}}
\def\SL{\hbox{\rm\tiny SL}}

\newtheorem{theorem}{Theorem}[section]
\newtheorem{proposition}[theorem]{Proposition}
\newtheorem{lemma}[theorem]{Lemma}

\newtheorem{remark}[theorem]{Remark}

\icmltitlerunning{Efficient Availability Attacks against  Supervised and Contrastive Learning Simultaneously}

\begin{document}
\twocolumn[
% \icmltitlerunning{Efficient Availability Attacks against  Supervised and Contrastive Learning Simultaneously}
\icmltitle{Efficient Availability Attacks against  Supervised and Contrastive Learning Simultaneously}

\begin{icmlauthorlist}
\icmlauthor{Yihan Wang}{yyy,xxx}
\icmlauthor{Yifan Zhu}{yyy,xxx}
\icmlauthor{Xiao-Shan Gao}{yyy,xxx}
\end{icmlauthorlist}

\icmlaffiliation{yyy}{Academy of Mathematics and Systems Science, Chinese Academy of Sciences, Beijing 100190, China}
\icmlaffiliation{xxx}{University of Chinese Academy of Sciences, Beijing 100049, China}
\icmlcorrespondingauthor{Xiao-Shan Gao}{xgao@mmrc.iss.ac.cn}

\icmlkeywords{Machine Learning, ICML}

\vskip 0.3in
]
\printAffiliationsAndNotice{}

\begin{abstract}
Availability attacks can prevent the unauthorized use of private data and commercial datasets by generating imperceptible noise and making unlearnable examples before release. 
Ideally, the obtained unlearnability prevents algorithms from training usable models. 
When supervised learning (SL) algorithms have failed, a malicious data collector possibly resorts to contrastive learning (CL) algorithms to bypass the protection.
Through evaluation, we have found that most of the existing methods are unable to achieve both supervised and contrastive unlearnability, which poses risks to data protection.
Different from recent methods based on contrastive error minimization, we employ contrastive-like data augmentations in supervised error minimization or maximization frameworks to obtain attacks effective for both SL and CL.
Our proposed AUE and AAP attacks achieve state-of-the-art worst-case unlearnability across SL and CL algorithms with less computation consumption, showcasing prospects in real-world applications.
% For example, the worst-case unlearnability of AUE on ImageNet-100 across an SL algorithm and four CL algorithms reaches $7.5\%$.
\end{abstract}

\section{Introduction}
\label{sec:intro}
Availability attacks \citep{biggio2018wild} add imperceptible perturbations to the training data, making the subsequently trained model unavailable. 
The motivations behind these attacks involve protecting private data and commercial datasets from unauthorized use.
% \cite{fowl2021preventing}. 
For example, a malicious data collector may gather selfies from social media apps into a facial image set. 
In this type of scenario, availability attacks provide tools to process user images before release, preserving legibility but impeding subsequent training. 
In recent years, various availability attacks have been proposed \citep{feng2019learning, huang2020unlearnable, fowl2021adversarial} against supervised learning (SL).

Meanwhile, contrastive learning (CL) allows people to extract meaningful features from unlabeled data in a self-supervised way. 
After subsequent linear probing or fine-tuning, CL algorithms have achieved comparable accuracy or even surpassed the performance of SL \citep{chen2020simple,chen2020improved,grill2020bootstrap, chen2021exploring}.
Thus, an unauthorized data collector can resort to contrastive learning algorithms to train a usable model when supervised learning has failed.
On one hand, most attacks designed for poisoning SL are ineffective against CL (refer to Table \ref{tab:contrastive-shortcut-align-uniform-gap} and Section \ref{subsec:existing-attacks-against-cl}).
It shed light on a potential issue of using availability attacks to protect data:
a malicious data collector can traverse both supervised and contrastive algorithms to effectively leverage collected data.
On the other hand, the error minimization poisoning framework has been extended to poison contrastive learning \citep{he2022indiscriminate}, and then components requiring label information have been incorporated into contrastive error minimization to simultaneously achieve supervised unlearnability besides contrastive unlearnability \cite{ren2022transferable, liu2023transferable}.
Compared to SL-based ones, these CL-based methods lack efficiency in poisoning generation, potentially hindering the use of availability attacks to protect extensive data in the real world (refer to Section \ref{subsec:efficiency}).

% {\red It's better to emphasize "efficient" is important in the beginning, which is the motivation of using SL to replace CL. Say "Because ... need efficient" here?}

%\begin{figure}[htbp]
\begin{figure*}[t]
\centering
\includegraphics[width=1.4\columnwidth]{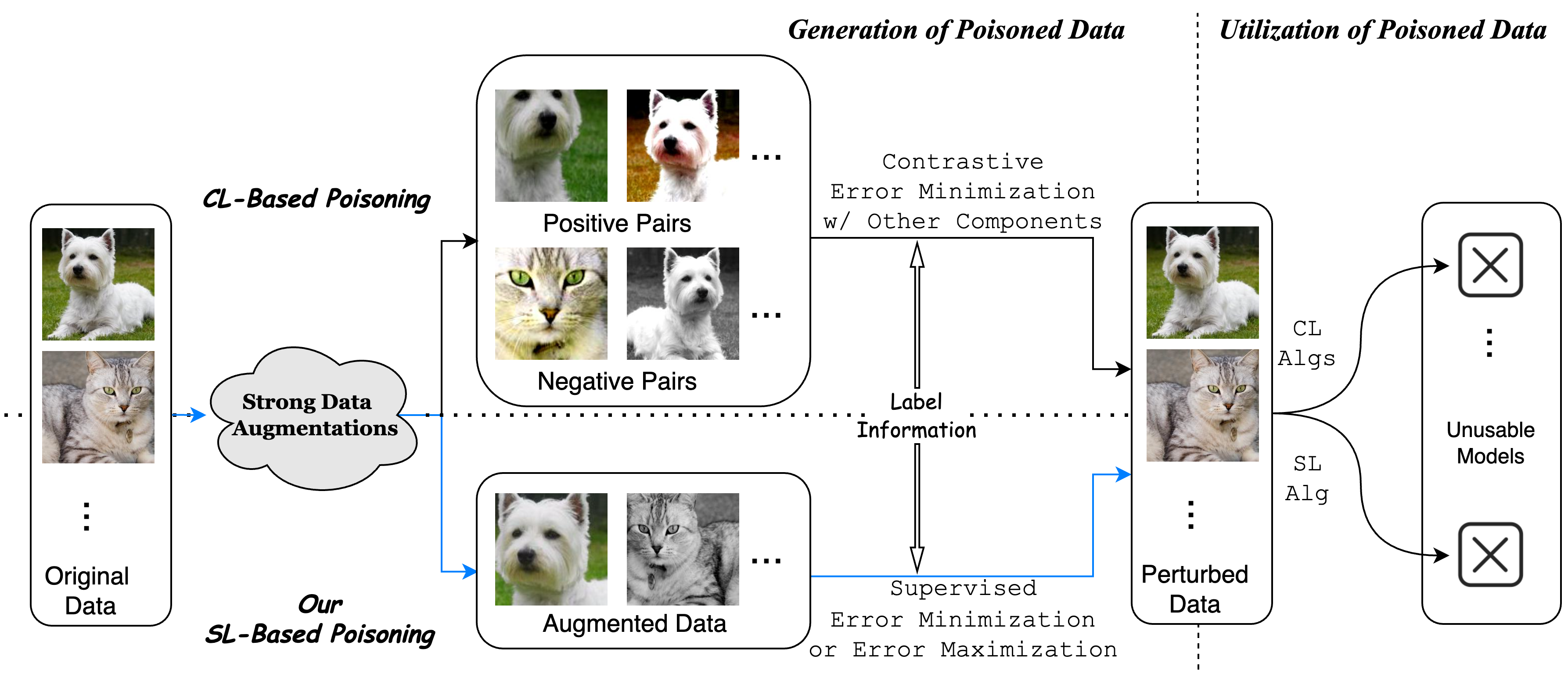}
    \caption{Illustration of our methods. The left-bottom flow (blue) is our supervised learning-based poisoning generation. The left-top flow is contrastive learning-based poisoning generation. The right flows are multiple supervised and contrastive learning evaluations on the poisoned data.}
    \label{fig:intro}
\end{figure*}

Our motivation for this paper comes from two aspects:
1)\textit{ A fully functional availability attack needs to be effective against subsequent supervised and contrastive learning algorithms simultaneously.}
2) \textit{Attacks based on supervised learning can be superior in efficiency compared to those based on contrastive learning.} 
To design non-CL-based availability attacks that possess both supervised and contrastive unlearnability simultaneously, we start from an interesting observation that supervised training with contrastive data augmentations can mimic contrastive training to some extent (refer to Section \ref{subsec:mimic-cl-with-supervised-models}). 
As shown in Figure \ref{fig:intro}, this technique of enhancing data augmentations can be easily embodied in two basic supervised attack frameworks, i.e. error-minimization, and error-maximization, resulting in our proposed AUE and AAP attacks (refer to Sections \ref{subsec:aue} and \ref{subsec:aap}).
Perturbations of our attacks are crafted on contrastive-like reference models and thus implicitly adapt to the contrastive training process and learn patterns that deceive contrastive learning. 
Besides, the supervised unlearnability is still preserved by the supervised error minimization or maximization framework.

We introduce the worst-case unlearnability in threat model (refer to Section \ref{subsec:threat-model}) to emphasize the demand for attacks to deal with a trickier unauthorized data collector.
In experiments, we evaluate across the standard supervised algorithm and four representative contrastive algorithms, SimCLR \citep{chen2020simple}, MoCo \citep{chen2020improved}, BYOL \citep{grill2020bootstrap} and SimSiam \citep{chen2021exploring}.
Our proposed AUE and AAP attacks achieve state-of-the-art worst-case unlearnability on CIFAR-10/100 and Tiny/Mini-ImageNet datasets (refer to Section \ref{subsec:worst-case-unlearnability}).
Furthermore, our methods exhibit excellent performance on the ImageNet-100, i.e. $7.5\%$ achieved by AUE, showcasing its prospects in real-world applications.
Meanwhile, unlike methods that add additional components to the contrastive error-minimization framework, we change the data augmentation in the simpler supervised attack frameworks, following a minimalist approach to algorithm design.  
Benefiting from this, our methods are more efficient, while delivering better performance.
For example, on CIFAR-10, our attacks are at least 3x, 6x, and 17x faster than three existing CL-based attacks (refer to Section \ref{subsec:efficiency}).
We summarize our contributions:
\vspace{-5pt}
\begin{itemize}
    \item We evaluated existing availability attacks and pointed out the potential security risks of using them to protect data when facing data abusers who will traverse both supervised and contrastive learning algorithms.
    \vspace{-5pt}
        \item 
        % Unlike recent methods based on contrastive error minimization, 
        We start from supervised poisoning approaches and enhance data augmentations to attain attacks against both supervised and contrastive learning.
    \vspace{-5pt}
    \item 
    Our attacks achieve state-of-the-art worst-case unlearnability with less computation consumption and are more adept at handling high-resolution datasets.
\end{itemize}

\section{Background and Related Works}
\subsection{Contrastive Learning}
Contrastive learning is self-supervised and does not require label information until linear probing or fine-tuning.
In general, it first augments an input into two views using augmentations sampled from a strong augmentation distribution $\mu$.
Then extracted features are trained to be aligned between positive pairs but distinct between negative pairs.
\citet{wang2020understanding} introduced two key properties for contrastive learning, \textit{alignment} and \textit{uniformity}.
The former measures the similarity of features from positive pairs and the latter reflects the uniformity of feature distribution on the hypersphere. 
Let $g$ be a normalized feature extractor.
The \textit{alignment loss} and \textit{uniformity loss} on a dataset $\Dd_c$ are defined as the following:
\begin{align*}
\A(\Dd_c)&=\E_{\boldsymbol{x}\sim \Dd_c\atop \pi, \tau\sim\mu} \big[ ||g(\pi(\boldsymbol{x})) - g(\tau(\boldsymbol{x}))||_2^2 \big],\\
\U(\Dd_c)&=\log \E_{\boldsymbol{x}, \boldsymbol{z} \sim \Dd_c\atop\pi, \tau\sim\mu} 
    \big[e^{-2||g(\pi(\boldsymbol{x})) - g(\tau(\boldsymbol{z}))||_2^2} \big].
\end{align*}
\vspace{-3pt}
Let $\Dd_c'$ be a poisoned version of a clean dataset $\Dd_c$.
The \textit{alignment gap} and \textit{uniformity gap} between clean and poisoned datasets are defined as follows:
\begin{align}
\label{equ:gaps}
    \AG=\A(\Dd_c)-\A(\Dd_c'), \ \ \ 
    \UG=\U(\Dd_c)-\U(\Dd_c').
    % \AG(\Dd_c, \Dd_c')=\A(\Dd_c)-\A(\Dd_c'),\\ 
    % \UG(\Dd_c, \Dd_c')=\U(\Dd_c)-\U(\Dd_c').
\end{align}
Intuitively, these gaps characterize the difference between clean features and poisoned features.
We will check the relationship between these gaps and contrastive unlearnability in Section \ref{subsec:existing-attacks-against-cl}.

\subsection{Basic Availability Attacks}
% \vspace{-1pt}
\label{subsec:basic-approaches}
The essence of availability attacks is to prevent a trained model from well generalizing to the clean data.
Error minimization and maximization are two representative approaches to poisoning supervised learning. 

\textbf{Error minimization.}
Unlearnable example attacks (UE, \citealt{huang2020unlearnable}) generate poisoning by alternately optimizing the reference model and perturbations:
% \vspace{-3pt}
\begin{align} \label{equ:supervised-error-minimizing}
    \min_{\delta} \min_f \E_{\Dd_c}\big[ \L_{\SL}(\boldsymbol{x}+{{\delta}}(\boldsymbol{x},y),y;f)\big],
\end{align}
where $f$ is a classifier, $\L_{\SL}(\cdot,\cdot;\cdot)$ is the supervised loss, $\Dd_c$ is a dataset to be processed and $\delta$ is a poisoning map.
% \vspace{-5pt}

\textbf{Error maximization.}
Adversarial poisoning attacks (AP, \citealt{fowl2021adversarial}) optimize perturbations through a pre-trained classifier to equip them with  non-robust but useful features from a different label:
% \vspace{-3pt}
\begin{align}\label{equ:adversarial-poisoning}
    &\min_{\delta} \E_{\Dd_c}\big[ \L_{\SL}(\boldsymbol{x}+{\delta}(\boldsymbol{x},y), y+K;f^*)\big]\textit{\ (Targeted)}\\
    \textit{or\ }& \max_{{\delta}} \E_{\Dd_c}\big[ \L_{\SL}(\boldsymbol{x}+{\delta}(\boldsymbol{x},y), y;f^*)\big]
        \textit{\ (Untargeted)}\nonumber\\
    &\text{s.t.\ \ \ \ } f^*\in \arg\min_f \E_{\Dd_c}\big[ \L_{\SL}(\boldsymbol{x},y;f)]\big].\nonumber
\end{align}
\citet{chen2023self} proposed self-ensemble protection (SEP-FA-VR) that generated adversarial poisons using several checkpoints to improve supervised unlearnability.

\textbf{Contrastive error minimization.}
To poison CL, the error minimization framework has been extended to contrastive error minimization (CP, \citealt{he2022indiscriminate}):
\begin{align} \label{equ:contrastive-error-minimizing}
    \min_{{{\delta}}} \min_g \E_{\Dd_c}\big[ \L_{\CL}(\boldsymbol{x}+{{\delta}}(\boldsymbol{x},y);g)],
\end{align}
where $g$ is a feature extractor and $\L_{\CL}(\cdot;\cdot)$ denotes the contrastive loss for simplicity.
% {\red does it matter if we omit augmentations in the loss function?}
% but in practice, it usually involves one positive sample and several negative samples \citep{oord2018representation}.
Later, a regularization term called class-wise separability discriminant was introduced to equip noises with supervised unlearnability (TUE,  \citealt{ren2022transferable}).
Then, \citet{liu2023transferable} combined contrastive error minimization with supervised adversarial poisoning to create transferable poisoning (TP). 
It is worth mentioning that both TUE and TP leverage label information to obtain supervised unlearnability, while CP requires no label information but lacks effect on supervised learning.

\subsection{Related Works}
Availability attacks against supervised learning also include 
\citet{yuan2021neural, feng2019learning, sandoval2022autoregressive, wu2022one, sadasivan2023cuda, Liu2024game}.
\citet{yu2022availability} illustrated linearly separable perturbations work as shortcuts for supervised learning.
% Moreover, \citet{zhang2023unlearnable} proposed cluster-wise perturbations in label-agnostic settings.
% \citet{Fu2022RobustUE, tao2022can,wen2023is} designed attacks to deceive adversarial training.
% Defense methods mainly include adversarial training \citep{tao2021better}, defensive augmentations \cite{liu2023image, qin2023learning, zhu2023detection}, diffusion purification \citep{jiang2023unlearnable, dolatabadi2023devil} and orthogonal projection \citep{sandoval2023can}.
Mild supervised data augmentation has been applied in poisoning generation \citep{fowl2021adversarial, Fu2022RobustUE}. 
Specially designed data augmentations were introduced as defense methods \cite{liu2023image, qin2023learning}.
To the best of our knowledge, we are the first to employ contrastive-like strong data augmentations in the generation of SL-based availability attacks.
Refer to Appendix \ref{app-sec:additional-realated-works} for additional related works.
\section{Threat Model}
In our threat model, we assume that an unauthorized data collector assembles labeled data into a dataset. The access to label information is reasonable since the collector can crawl individual images from certain accounts or steal (and annotate) a commercial dataset. 
A data publisher is supposed to process data before release using an availability attack such that processed data is resilient to subsequent supervised learning algorithms as well as contrastive learning algorithms adopted by the data collector. 
\subsection{Worst-Case Unlearnability}
\label{subsec:threat-model}
Suppose a dataset $\Dd_c$ to be processed is \textit{i.i.d} sampled from a data distribution $\Dd$.
An availability attack $\delta$ maps a data-label pair $(\boldsymbol{x},y)\in \Dd_c$ to a noise $\delta(\boldsymbol{x},y)$ within an $L_p$-norm ball $\B_p(\epsilon)$.
In this paper, we set $p=\infty$ and $\epsilon=8/255$.
It results in a protected dataset $\Dd_c'=\{(\boldsymbol{x}+\delta(\boldsymbol{x},y), y)|(\boldsymbol{x},y)
\in \Dd_c\}$ to which a data collector has only access.
For potential algorithms, we refer $f$ to a supervised model and $g$ to a contrastive feature extractor beyond which is a linear probing head $h$. 
The goal of the data publisher is to find a poisoning map $\delta$ that significantly degrades the generalization performance of both $f_\delta$ and $h_\delta\circ g_\delta$ which are trained on $\Dd_c'$.
In this paper, we consider the \textit{worst-case unlearnability across supervised and contrastive learning algorithms} of the following form:
% \vspace{-1pt}
\begin{align}
\label{equ:threat-model-worst-case}
\min_{{\delta}} \max&( \E_{\Dd}\big[ \textbf{1}(f_\delta(\boldsymbol{x})= y) \big], \E_{\Dd}\big[\textbf{1}(h_\delta\circ g_\delta(\boldsymbol{x})= y) \big] )\\
\text{s.t. \ \ \ \ }
&f_\delta \in \arg\min_f 
\E_{\Dd_c}\big[\L_{\SL}(\boldsymbol{x}+{\delta}(\boldsymbol{x},y),y;f)\big], 
\nonumber\\
&g_\delta \in \arg\min_g \E_{\Dd_c}\big[\L_{\CL}(\boldsymbol{x}+{\delta}(\boldsymbol{x},y);g)\big], 
\nonumber\\
&h_\delta \in \arg\min_h \E_{\Dd_c}\big[\L_{\SL}(\boldsymbol{x}+{\delta}(\boldsymbol{x},y),y;h\circ g_\delta)\big].
\nonumber
\end{align}
%\vspace{-3pt}
It is a fair metric that accurately depicts scenarios facing more cunning data abusers in reality. 
In contrast, other metrics, such as average-case unlearnability, can be heavily influenced by the attack's strong preference for a certain algorithm.
Our threat model differs from the setting adopted by \citet{he2022indiscriminate} in which the linear probing stage relies on the unprocessed clean data as downstream tasks; see more discussion in Appendix \ref{app-subsect: discussion-of-clean-linear-probing}.

\subsection{Existing Attacks against Contrastive Learning}
\label{subsec:existing-attacks-against-cl}
In Table \ref{tab:contrastive-shortcut-align-uniform-gap}, we evaluate the attack performance of existing poisoning approaches against the SimCLR algorithm on CIFAR-10 and ResNet-18.
To better understand contrastive unlearnability, we also check alignment and uniformity gaps between clean and poisoned data which are defined in Eq.\eqref{equ:gaps}.
For non-CL-based poisoning attacks, AP and SEP achieve high contrastive unlearnability while others fail to deceive the contrastive learning algorithm.
The alignment and uniformity gaps of AP and SEP attacks are prominently larger than those of others.
For CL-based attacks, CP, TUE, and TP are effective against contrastive learning and possess huge alignment and uniformity gaps.

\begin{table}[t]
\centering
\caption{
Alignment and uniformity gaps of poisoned SimCLR \citep{chen2020simple} models along with the test accuracy. 
Attacks are grouped according to whether they are based on contrastive error minimization.
\textbf{Bold} fonts emphasize prominent contrastive unlearnability values.
}
\label{tab:contrastive-shortcut-align-uniform-gap}
\resizebox{\columnwidth}{!}{
\begin{tabular}{l|ccc}
\toprule
           Attack   & $\AG$    & $\UG$     &  Accuracy(\%) \\
\midrule
DC~\citep{feng2019learning}         & 0.12    & 0.07    & 86.1            \\
UE~\citep{huang2020unlearnable}          & 0.05   & 0.03    & 89.0      \\
AR~\citep{sandoval2022autoregressive}           & 0.07   & 0.09    & 88.8       \\
NTGA~\citep{yuan2021neural}        & 0.12   & 0.12    & 86.9          \\
SN~\citep{yu2022availability}         & 0.08   & 0.00    & 90.6       \\
OPS~\citep{wu2022one}         &0.04  &0.01 &86.7    \\
GUE~\cite{Liu2024game} &0.07&0.03&88.8\\
REM~\citep{Fu2022RobustUE}         &0.12  &0.04 &88.6    \\
EntF~\cite{wen2023is}        & 0.01   & -0.04   & 87.5    \\
HYPO~\citep{tao2022can}        & 0.11   & 0.13   & 86.9  \\
T-AP~\cite{fowl2021adversarial}& \textbf{0.18}   & \textbf{0.44}    & \textbf{48.4}      \\
UT-AP~\citep{fowl2021adversarial}& \textbf{0.17}   & \textbf{0.77}    & \textbf{41.5}      \\
SEP-FA-VR~\cite{chen2023self}  &\textbf{0.24}   &\textbf{0.25}  & \textbf{37.3}       \\
\hdashline
CP~\cite{he2022indiscriminate}   & \textbf{0.55}   & \textbf{0.87}    & \textbf{38.7}        \\
TUE~\cite{ren2022transferable}         & \textbf{0.30}  & \textbf{0.76}    & \textbf{48.1}        \\
TP~\cite{liu2023transferable} &\textbf{0.52}&\textbf{0.82}&\textbf{31.4}\\
\bottomrule
\end{tabular}
}
\end{table}

The Pearson correlation coefficient (PCC) between the alignment gap and the SimCLR accuracy is $-0.82$, and the PCC between the uniformity gap and the SimCLR accuracy is $-0.88$. 
It reveals that contrastive unlearnability is highly related to huge alignment and uniformity gaps, which indicate a significant difference between clean feature distribution and poisoned feature distribution.
% The gaps work as a ``shortcut'' for CL. 
After poisoned contrastive training, the feature extractor is fixed and the linear probing stage trains a linear layer to classify poisoned features.
If the gaps are huge, even though the extracted features of poisoned data are highly linear separable, the learned separability can hardly be generalized to clean features due to the huge discrepancy between clean features and poisoned features.
Consequently, even if the accuracy of poisoned data is high, the accuracy of clean data is likely to be still low and the poisoning attack succeeds.
In contrast, small gaps likely imply that clean features are similar to poisoned features. 
Therefore, once the classifier performs correct classification on poisoned data, it can generalize to clean data and thus the attack fails.

\section{Method}
Contrastive error minimization (CP) directly optimizes the contrastive loss on poisoned data which relates to alignment and uniformity.
%Based on this, TUE and TP include additional components to obtain supervised unlearnability.
Based on this, TUE and TP involve CP with additional components to obtain supervised unlearnability.
Since optimizing contrastive loss is very time and memory-consuming, we start from supervised frameworks instead and try to implicitly optimize the contrastive loss during the poisoning generation.
The key point to achieve our goal is to enhance data augmentations.
In the rest of this section, we first illustrate how contrastive data augmentations help mimic contrastive learning with supervised models through empirical observations and intuition from a toy example.
Then we combine this very effective technique with supervised error minimization and maximization frameworks and propose \textit{augmented unlearnable examples (AUE) attacks} and \textit{augmented adversarial poisoning (AAP) attacks}.

\subsection{Mimic Contrastive Learning with Supervised Models}
\label{subsec:mimic-cl-with-supervised-models}
Contrastive learning employs strong data augmentations including \textit{resized crop, color jitter, horizontal flip, and grayscale} \cite{wu2018unsupervised,he2020momentum}, while supervised learning adopts mild data augmentations such as horizontal flip and crop.
In Appendix \ref{app-subsec:augmentations}, Code \ref{code} shows detailed implementations for these two different settings.
The contrastive error minimization optimizes the contrastive loss involving strong contrastive augmentations while the supervised error minimization or maximization optimizes the supervised loss involving mild supervised augmentations.
What if we use strong contrastive augmentations when computing supervised losses?
\vspace{-5pt}
% \begin{figure}[htbp]
%     \centering
%     \includegraphics[width=0.7\columnwidth]{figures/ce-infonce.eps}
%     \caption{InfoNCE loss decreases with CE loss.}
%     \label{fig:supervised-training-with-contrastive-augmentations}
% \end{figure}
\begin{figure}[!ht]
\centering
\includegraphics[width=0.32\textwidth]{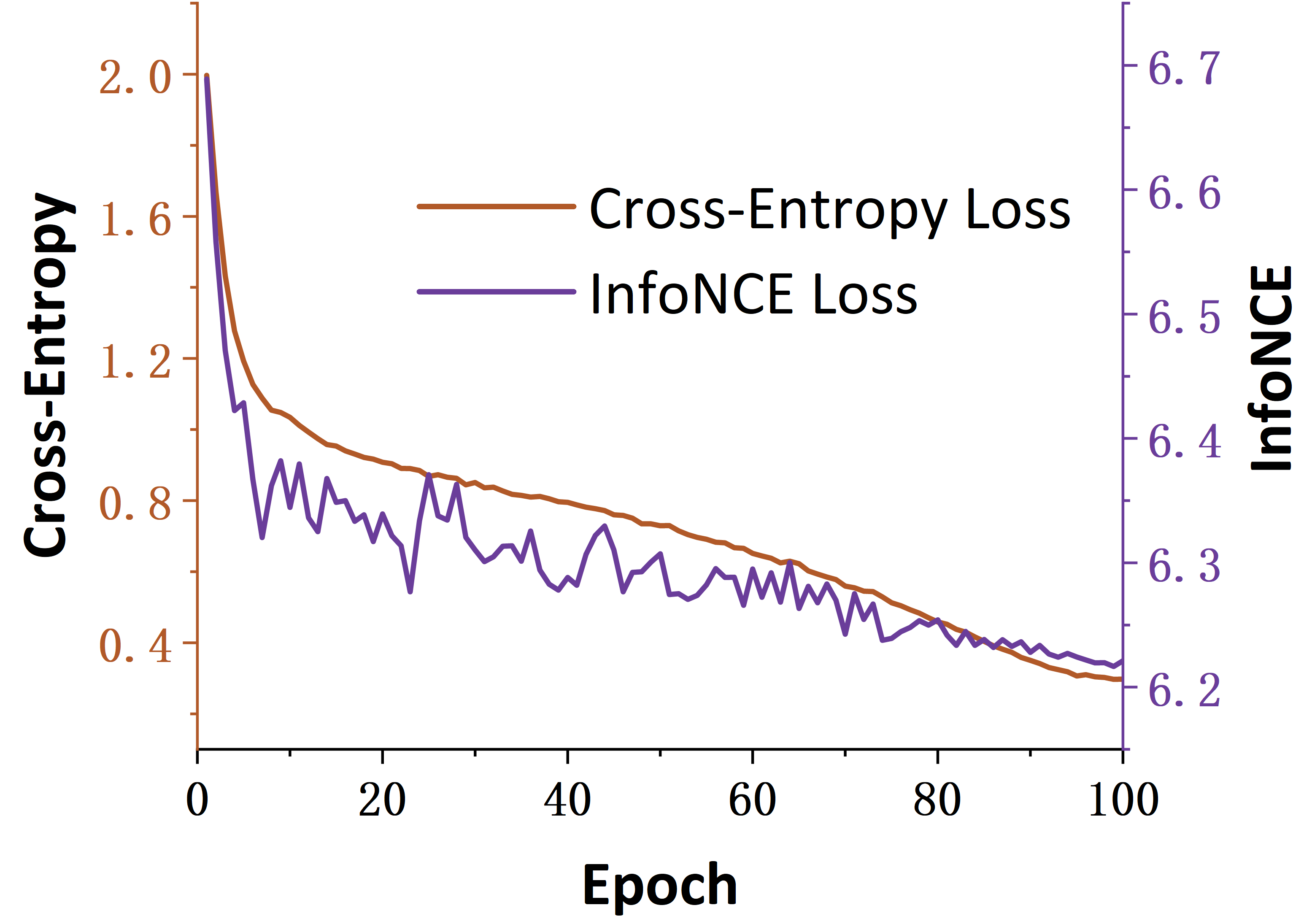}
    \caption{InfoNCE loss decreases with CE loss.}
    \label{fig:supervised-training-with-contrastive-augmentations}
\end{figure}
On CIFAR-10, we train a supervised ResNet-18 using contrastive augmentations.
For each checkpoint, the supervised CE loss and the contrastive InfoNCE loss \citep{oord2018representation} are computed on the training set.
In Figure \ref{fig:supervised-training-with-contrastive-augmentations}, when the optimization object CE loss goes down, the InfoNCE loss decreases as well.
It indicates that training a supervised model with contrastive augmentations implicitly optimizes the contrastive loss.
Therefore, incorporating stronger data augmentation can potentially enable availability attacks based on supervised error minimization or maximization to acquire the ability to deceive contrastive learning.

To provide more intuition about this idea, we give a toy example and have a closer look at the relationship between supervised loss and contrastive loss.
For a supervised model $f=h\circ g$, assume $g$ is a normalized feature extractor, $h$ is a square full-rank linear classifier, $\Dd$ is a balanced distribution, $\L_{\SL}$ is MSE loss and $\Ee_{\SL}=\E\big[\L_{\SL}\big]$, and  $\L_{\CL}$ contains only one negative example.
In this toy example, if $\L{\CL}$ and $\L{\SL}$ employ the same data augmentation and $f$ is well-trained, it holds with high probability that $\L_{\CL}<l(\Ee_{\SL})$, where $l(\cdot)$ is an increasing function.
In other words, the upper bound of contrastive loss decreases as the supervised loss decreases.
We have a more detailed and formal discussion on this toy example in Appendix \ref{app-sec:proofs}.

% To , we study a toy model $f=h\circ g$ where $g$ is a feature extractor and $h$ is a linear classifier.
% Let $\Dd$ be a data distribution, $\L_{\SL}$ be the supervised loss, $\L_{\CL}$ be the contrastive loss, and  $\Ee_{\SL}=\E\big[\L_{\SL}\big]$ be the supervised risk.
% Here we introduce some conditions to this toy example and then present an informal proposition for illustration. 
% The formal version of the conditions and this proposition is shown in Appendix \ref{app-sec:proofs}.
% Assume $g$ is normalized, $h$ is square and full-rank, $\Dd$ is balanced,  $\L_{\CL}$ contains only one negative example, and $\L_{\SL}$ is the mean squared error.
% The following claim says the upper bound of contrastive loss decreases as the supervised loss decreases in a range of values if their data augmentations obey the same distribution.
% \begin{proposition}[Informal] \label{thm:informal-main}
% In this toy example, if $\L{\CL}$ and $\L{\SL}$ employ the same data augmentation,  for a well-trained supervised model, it holds with high probability that 
% $\L_{\CL}<l(\Ee_{\SL})$, where $l(\cdot)$ is an increasing function.
% \end{proposition}

Based on these interesting observations, instead of adding components to contrastive error minimization to achieve supervised unlearnability, we opt for deriving stronger contrastive unlearnability from supervised error minimization and maximization.

\subsection{Augmented Unlearnable Examples (AUE)}
\label{subsec:aue}

\begin{algorithm*}[htb]
\caption{Augmented Unlearnable Examples (AUE)}
\label{alg:aue}
\begin{algorithmic}
\Require Augmentation strength $s$ and a corresponding augmentation distribution $\mu_s$.
    A labeled training set $\Dd_c = \{(\boldsymbol{x}_i,y_i)\}_{i=1}^{r}$. 
    An initialized classifier $f_\theta$.
    Total epochs $T$, model update iterations $T_\theta$, 
    poisons update iterations $T_\delta$, and perturbation steps $T_p$.
    Learning rate $\alpha_\theta, \alpha_\delta$.
\Ensure Perturbations $\{\boldsymbol{\delta}_i\}_{i=1}^{r}$

\State $\boldsymbol{\delta}_i\gets 0, i=1,2,\cdots,r$ \Comment{Initialize perturbations}

\For{$t=1,\cdots,T$}
    \For{$t_\theta = 1,\cdots, T_\theta$} \Comment{Update the reference model}
        \State Sample a data batch $\{(\boldsymbol{x}_{l_j}, y_{l_j})\}_{j=1}^m$ and an augmentation batch $\{\pi_{l_j}\sim\mu_s\}_{j=1}^m$
        \State $\theta \gets \theta - \frac{\alpha_\theta}{m} \cdot \sum_{j=1}^m \nabla_\theta \L_{\SL}(\pi_{l_j}(\boldsymbol{x}_{l_j}+
        \boldsymbol{\delta}_{l_j}), y_{l_j};f_\theta)$ 
    \EndFor
    \For{$t_\delta = 1,\cdots, T_\delta$} \Comment{Update perturbations}
        \State Sample a data batch $\{(\boldsymbol{x}_{l_j}, y_{l_j})\}_{j=1}^m$
        \For{$t_p = 1,\cdots, T_p$}
            \State Sample an augmentation batch $\{\pi_{l_j}\sim \mu_s\}_{j=1}^m$
            \State $\boldsymbol{\delta}_{l_j} \gets \text{Clip}_\epsilon 
            \big( \boldsymbol{\delta}_{l_j} - \alpha_{\delta} \cdot \text{sign}(\nabla_{\boldsymbol{\delta}_{l_j}} \L_{\SL}(\pi_{l_j}(\boldsymbol{x}_{l_j}+\boldsymbol{\delta}_{l_j}), y_{l_j};f_\theta)) \big), 
            j=1,2,\cdots,m$
        \EndFor
    \EndFor
\EndFor
\end{algorithmic}
\end{algorithm*}
Recall unlearnable examples are generated by supervised error minimization in which a reference model and noises alternately update in Eq.\eqref{equ:supervised-error-minimizing}.
Now we employ contrastive-like strong data augmentations $\pi\sim \mu$ and add perturbations in a differentiable way, i.e. $\pi(\boldsymbol{x}+{\delta}(\boldsymbol{x},y))$.
Then minimizing the augmented supervised loss $\L_{\SL}(\pi(\boldsymbol{x}+{\delta}(\boldsymbol{x},y),y;f)$ implicitly minimizes the contrastive loss $\L_{\CL}(\boldsymbol{x}+{\delta}(\boldsymbol{x},y);g)$ which appears in contrastive error minimization, i.e. Eq.\eqref{equ:contrastive-error-minimizing}.
In other words, supervised error-minimizing noises with enhanced data augmentations can partially replace the functionality of contrastive error-minimizing noises to deceive contrastive learning.
%\vspace{-3pt}

\begin{figure}[b]
     \centering
     \begin{subfigure}[b]{0.23\textwidth}
         \centering
         \includegraphics[width=\textwidth]{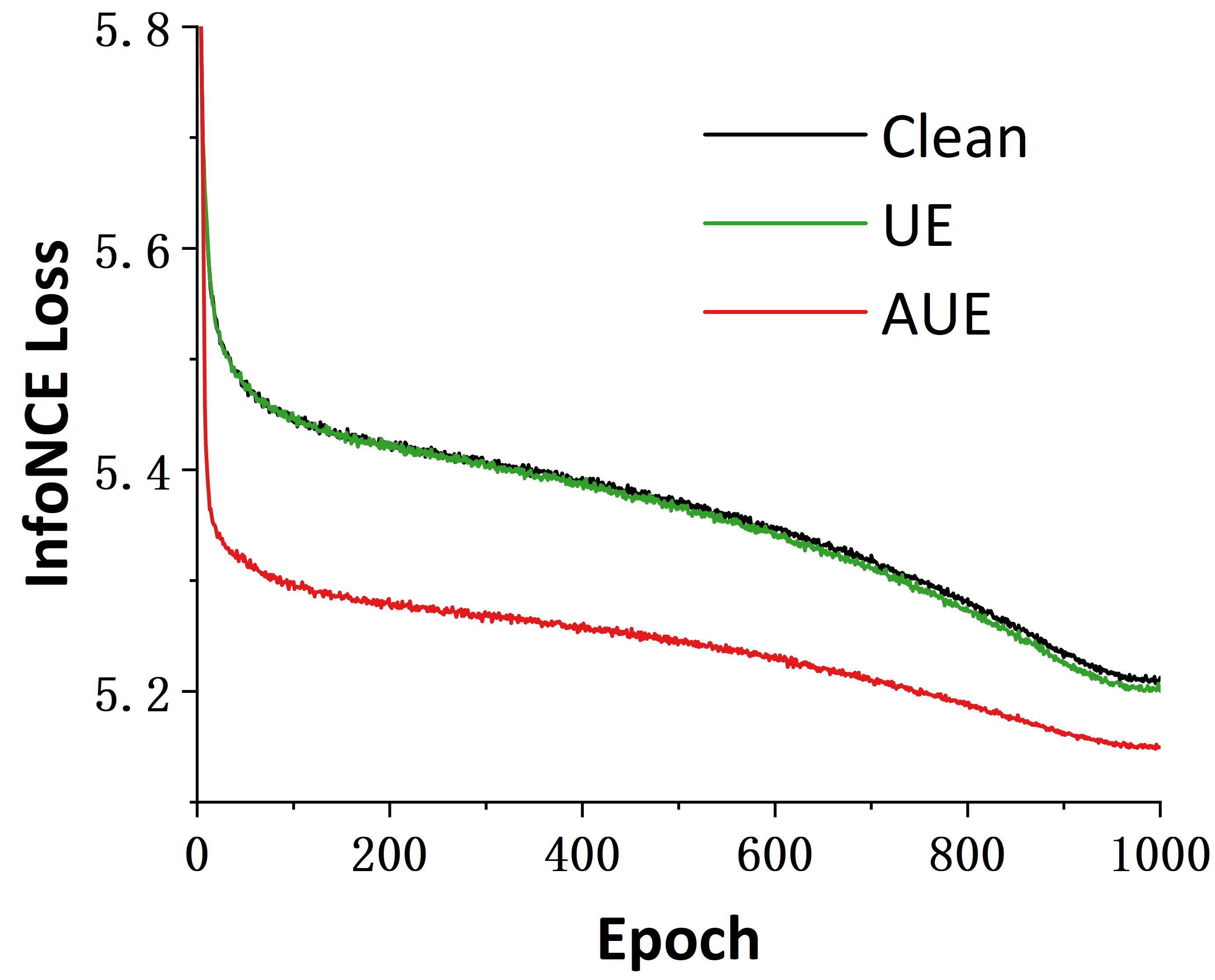}
         \caption{}
         \label{subfig:aue-infonce-loss}
     \end{subfigure}
     \hfill
     \begin{subfigure}[b]{0.23\textwidth}
         \centering
         \includegraphics[width=\textwidth]{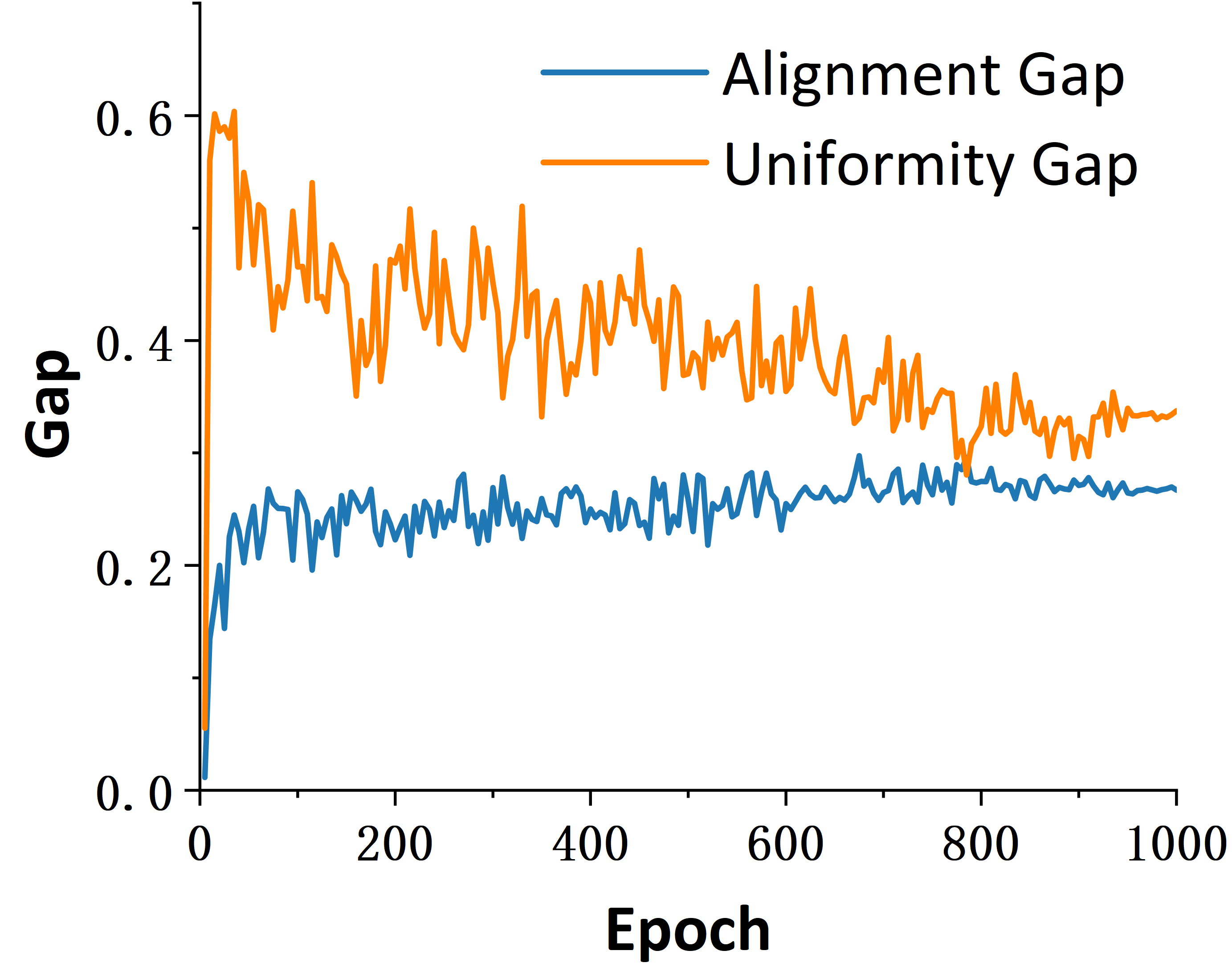}
         \caption{}
         \label{subfig:gaps-of-AUE}
     \end{subfigure}
    \caption{
    (a) Contrastive losses during SimCLR training under UE and our AUE attacks. 
    (b) Alignment and uniformity gaps during the SimCLR training on CIFAR-10 poisoned by our AUE attack.  }
\end{figure}

According to Code \ref{code} in Appendix \ref{app-subsec:augmentations}, we can control the intensity of contrastive augmentations via a strength hyperparameter $s\in [0,1]$. 
We increase the augmentation strength in the supervised error minimization according to Algorithm \ref{alg:aue} and check the performance of proposed AUE attacks against contrastive learning.
Details of strength parameter selection are shown in Appendices \ref{app-subsec:details-of-aue-aap} and \ref{app-subsec:strength-selection}.
In Table \ref{tab:simclr-accuracy-drop}, while UE attacks do not work for SimCLR on CIFAR-10 and CIFAR-100, our AUE attacks successfully reduce the SimCLR accuracy by $38.9\%$ and $50.3\%$.
Enhanced data augmentations indeed make supervised error-minimizing noises effective for contrastive learning. 
In Figure \ref{subfig:aue-infonce-loss}, AUE noises largely reduce the contrastive loss during SimCLR training compared to UE noises. 
% It verifies that supervised error minimization with enhanced augmentations mimics contrastive error minimization to some extent.
In Figure \ref{subfig:gaps-of-AUE}, we investigate the alignment and uniformity gaps and discuss more about the poisoned training process in Section \ref{subsec:poisoned-training-process}.
The final gaps of AUE are $\AG=0.27, \UG=0.34$ while those of UE are $\AG=0.05, \UG=0.03$.

% \vspace{-5pt}

\begin{table}[ht]
\centering
\caption{The accuracy drop(\%) of SimCLR caused by basic attacks and our proposed methods.}
\label{tab:simclr-accuracy-drop}
\resizebox{\columnwidth}{!}{
\begin{tabular}{lc|cc|cc}
\toprule
          & Clean & UE   & AUE (ours)  & T-AP & T-AAP (ours) \\
\midrule
CIFAR-10  & 91.3  & -2.3 & -38.9 & -42.9
 & -52.2
  \\
CIFAR-100 & 63.9  & -3.9 & -50.3 & -38.3& -43.8
  \\
\bottomrule
\end{tabular}
}
\end{table}
% \vspace{-5pt}

\subsection{Augmented Adversarial Poisoning (AAP)}
\label{subsec:aap}
\begin{algorithm*}[t]
\caption{Augmented Adversarial Poisoning (AAP)}
\label{alg:aap}
\begin{algorithmic} 
\Require Similar to the setting in Algorithm \ref{alg:aue}.
\Ensure Perturbations $\{\boldsymbol{\delta}_i\}_{i=1}^{r}$
\State $\boldsymbol{\delta}_i\gets 0, i=1,2,\cdots,r$ \Comment{Initialize perturbations}
\For{$t=1,\cdots,T$} \Comment{Update the reference model}
    \For{$t_\theta = 1,\cdots, T_\theta$}
        \State Sample a data batch  $\{(\boldsymbol{x}_{l_j}, y_{l_j})\}_{j=1}^m$ and an augmentation batch $\{\pi_{l_j}\sim\mu_s\}_{j=1}^m$
        \State $\theta \gets \theta - \frac{\alpha_\theta}{m} \cdot \sum_{j=1}^m \nabla_\theta \L_{\SL}(\pi_{l_j}(\boldsymbol{x}_{l_j}), y_{l_j};f_\theta)$
    \EndFor
\EndFor
\For{$i=1,\cdots,r$} \Comment{Update adversarial examples}
    \For{$t_p = 1,\cdots, T_p$}
            \State Sample $\pi_i\sim \mu_s$
            \State $\boldsymbol{\delta}_i \gets \text{Clip}_\epsilon 
            \big( \boldsymbol{\delta}_i + \alpha_{\delta} \cdot \text{sign}(\nabla_{\boldsymbol{\delta}_i} \L_{\SL}(\pi_i(\boldsymbol{x}_i+\boldsymbol{\delta}_i), y_i;f_\theta)) \big)$\Comment{Untargeted}
            \State $\boldsymbol{\delta}_i \gets \text{Clip}_\epsilon 
            \big( \boldsymbol{\delta}_i - \alpha_{\delta} \cdot \text{sign}(\nabla_{\boldsymbol{\delta}_i} \L_{\SL}(\pi_i(\boldsymbol{x}_i+\boldsymbol{\delta}_i), y_i+1;f_\theta)) \big)$\Comment{Targeted}
    \EndFor
\EndFor
\end{algorithmic}
\end{algorithm*}

Adversarial poisoning (AP) attacks in Equ. \eqref{equ:adversarial-poisoning} first train a supervised reference model, then generate adversarial examples on it.
For targeted AP, while reference model training uses standard supervised loss,  the loss for noise generation translates class labels by $K$ such that generated poisons contain non-robust features that are related to the shifted labels.
When we generate adversarial poisoning with enhanced data augmentations $\pi \sim \mu$,
minimizing $\E\big[\L_{\SL}(\pi(\boldsymbol{x}), y; f)\big]$ with respect to $f$ mimics updating a reference model with contrastive training.
Then, minimizing $\L_{\SL}(\pi(\boldsymbol{x}+{\delta}(\boldsymbol{x},y)), y+K;f^*)$ with respect to ${\delta}$ updates poisons to deceive a contrastive-like reference model $f^*$.
For untargeted AP, stronger data augmentations play a similar role.
As a consequence, the resulting poisons can learn more about how to confound contrastive learning algorithms.

According to Algorithm \ref{alg:aap}, we increase the augmentation strength $s$ in both reference model pre-training and noise update where the label translation $K=1$.
Details of strength parameter selection are shown in Appendix \ref{app-subsec:details-of-aue-aap} and \ref{app-subsec:strength-selection}.
In Table \ref{tab:simclr-accuracy-drop}, our targeted AAP (T-AAP) attacks further enlarge the SimCLR accuracy drop of targeted AP by $9.3\%$ on CIFAR-10 and $5.5\%$ on CIFAR-100.
Enhanced data augmentations indeed improve the contrastive unlearnability of supervised error-maximizing noises. 

\section{Experiments}\label{sec:experiments}
\subsection{Setup}\label{subsec:settigns}
Poisons are generated on CIFAR-10/100, Tiny-ImageNet, modified Mini-ImageNet, and ImageNet-100. 
ResNet-18 \cite{he2016deep} is used for poison generation and evaluation if not otherwise stated.
Our threat model considers the worst-case unlearnability across supervised and contrastive algorithms. 
The standard supervised learning algorithm and four contrastive learning algorithms including 
%SimCLR \citep{chen2020simple}, MoCo  \citep{chen2020improved}, BYOL \citep{grill2020bootstrap} and SimSiam \citep{chen2021exploring}
SimCLR, MoCo, BYOL and SimSiam 
are employed to evaluate the attack performance of availability attacks.
We implement a linear probing stage on the poisoned data.

Due to Table \ref{tab:contrastive-shortcut-align-uniform-gap}, we adopt AP, SEP-FA-VR, CP, TUE, and TP as baselines for the worst-case unlearnability.
T-AP and T-AAP are targeted attacks and UT-AP and UT-AAP are untargeted. 
Since the generation of untargeted adversarial poisoning is unstable \citep{fowl2021adversarial}, we generate UT-AAP only on CIFAR-10.
Our AUE and AAP train reference models from scratch rather than using pre-trained weights.
For CP and TUE attacks, we specify algorithms they used for noise generation, for example, CP-SimCLR.
The generation of TP attacks is based on the SimCLR algorithm.
Detailed settings for evaluations and our proposed attacks are shown in Appendix \ref{app-sec:experiment-details}.

\begin{table}[htb]
\centering
\caption{Attack performance(\%) of availability attacks evaluated by supervised and contrastive learning algorithms on CIFAR-10. Attacks are grouped according to whether they are based on contrastive error minimization.}
\label{tab:worst-case-cifar10}
\resizebox{\columnwidth}{!}{
\begin{tabular}{l|c|cccc|c}
\toprule
Attack & SL & SimCLR & MoCo  & BYOL & SimSiam & Worst\\
\midrule
Clean          & 95.5 & 91.3  & 91.5   & 92.3 & 90.7   & 95.5 \\
\hdashline
T-AP    & 9.5  & 48.4   & 53.8    & 53.0 & 51.1    & 53.8  \\
UT-AP  & 9.6  & 41.5   & 31.5    & 44.0 & 42.8    & 44.0  \\
SEP-FA-VR &\textbf{2.3} &37.3  &35.8   &42.8 & 36.7    & 42.8  \\
\hdashline
CP-SimCLR     & 94.5 & 38.7   & 69.3    & 79.5 & \textbf{29.2}& 94.5  \\
CP-MoCo        & 94.5 & 53.7   & 47.9    & 56.8 & 47.1    & 94.5  \\
CP-BYOL       & 11.0 & 39.3   & 32.7    & 41.8 & 37.9    & 41.8  \\
TUE-SimCLR     & 10.6 & 48.1   & 71.2    & 79.5 & 39.0    & 79.5  \\
TUE-MoCo       & 10.1 & 57.2   & 51.6    & 60.1 & 58.5    & 60.1  \\
TUE-SimSiam    & 9.9  & 82.5   & 80.7    & 84.3 & 81.8    & 84.3  \\
TP &14.8&\textbf{31.4}&54.1&61.8&30.7&61.8 \\
\hdashline
AUE (ours)     & 18.9 & 52.4   & 57.0    & 58.2 & 34.5    &58.6  \\
T-AAP (ours)   &9.2  & 39.1   & 40.4    &43.3 & 42.1    & 43.3  \\
UT-AAP (ours) & 29.7 & 32.3   &  \textbf{23.2}  &\textbf{35.5} & 34.1   & \textbf{35.5} \\
\bottomrule
\end{tabular}
}
\end{table}

\subsection{Worst-Case Unlearnability}
\label{subsec:worst-case-unlearnability}
{\textbf{Performance.}} 
We evaluate standard supervised learning and four contrastive learning algorithms on CIFAR-10 and CIFAR-100 in Tables \ref{tab:worst-case-cifar10} and \ref{tab:worst-case-cifar100}.
In terms of the worst-case unlearnability across these evaluation algorithms, our proposed attacks achieve state-of-the-art attack performance compared to existing baselines. 
In CL-based methods, the CP methods have negligible unlearnability for supervised learning on the CIFAR-100 dataset, whereas the TUE-MoCo achieves performance close to the SOTA attack AUE. 
Thus, we then compare the attack performance of AUE and AAP with AP and TUE-MoCo baselines on high-resolution datasets in Table \ref{tab:imagenet}.
On Tiny-ImageNet and Mini-ImageNet, our methods improve by $13.4\%$ and $38.7\%$ over the baseline methods.
On ImageNet-100 whose images are 224x224, our AUE attacks reduce the worst-case unlearnability to a surprising level, $7.5\%$.
In real-world applications, our method is more promising because CL-based methods involve optimizing the contrastive loss, which is a significant challenge on high-resolution images such as ImageNet, whereas the supervised loss with which our SL-based methods deal is much easier.

Additionally, in a comparison between our two attacks, AAP performs better on datasets with fewer classes and lower resolution, while AUE is more effective on datasets with more classes and higher resolution.

\textbf{Algorithm transferability.}
CL-based methods face the issue of transferability from the generation CL algorithm and the evaluation CL algorithm. 
For instance, CP, TUE, and TP attacks generated using SimCLR are very effective against SimCLR itself, but their effectiveness significantly decreases when using BYOL and MoCo for training.

In contrast, our SL-based attacks get rid of this issue because their poisoning generation involves no CL algorithms. 
The contrastive unlearnability of AUE and AAP attacks is more stable across different CL algorithms.

\begin{table}[t]
\centering
\caption{Attack performance(\%) on CIFAR-100.}
\label{tab:worst-case-cifar100}
\resizebox{\columnwidth}{!}{
\begin{tabular}{l|c|cccc|c}
\toprule
Attack & SL & SimCLR & MoCo  & BYOL & SimSiam & Worst\\
\midrule
Clean          & 77.4 & 63.9   & 67.9    & 63.7 & 64.4    & 77.4  \\
\hdashline
T-AP    & 3.2  & 25.6   & 26.6    & 26.1 & 28.8    & 28.8  \\
UT-AP  & 42.7 & 11.1   & \textbf{9.8}   & \textbf{10.1} & 14.0    & 42.7  \\
SEP-FA-VR  & 2.4 & 25.2   & 25.9   & 26.6 & 28.4   &28.4  \\
\hdashline
CP-SimCLR     & 74.7 & 10.5   & 30.7    & 22.6 & 7.7     & 74.7  \\
CP-MoCo       & 74.4 & 15.2   & 13.4    & 16.4 & 14.1    & 74.4  \\
CP-BYOL       & 74.7 & 29.7   & 35.5    & 35.7 & 29.5    & 74.7  \\
TUE-SimCLR     & \textbf{1.0}  & 16.9   & 36.7    & 40.6 & 7.8     & 40.6  \\
TUE-MoCo      & 1.0  & 19.9   & 19.6    & 22.3 & 18.6    & 22.3  \\
TUE-SimSiam    & 1.1  & 33.9   & 31.0    & 40.9 & 10.3    & 40.9  \\
TP &7.5&\textbf{6.7}&21.9&27.0&\textbf{4.1}&27.0 \\
\hdashline
AUE (ours)     &6.9     & 13.6 &  19.0  & 19.2      & 11.9    & \textbf{19.2}  \\
T-AAP (ours)   & 7.3  & 20.1   & 18.6   &  21.1 & 21.3    & 21.3  \\
\bottomrule
\end{tabular}
}
\end{table}

\begin{table}[t]
\centering
\caption{Attack performance(\%) on higher-resolution datasets: Tiny-ImageNet (T-I), Mini-ImageNet (M-I), and ImageNet-100 (I-100).}
\label{tab:imagenet}
\resizebox{\columnwidth}{!}{
\begin{tabular}{cl|c|cccc}
\toprule
Dataset & Attack & SL & SimCLR & MoCo & BYOL & SimSiam \\
\midrule
\multirow{5}{*}{T-I}&Clean    & 53.5 & 39.6   & 43.3    &33.9  & 42.4   \\
&T-AP     & 11.3    &32.8 & 34.7  & 27.2    &34.5  \\
&TUE-MoCo   & \textbf{5.5}   &20.9 & 24.9  & 20.3    &25.0\\
&AUE (ours)     & 7.1   &\textbf{10.8}  & \textbf{11.7}  & \textbf{9.6 }    &\textbf{11.6}  \\
&T-AAP (ours)   & 18.7 & 28.4 & 27.6  & 25.2 & 28.2   \\
\midrule
\multirow{5}{*}{M-I}&Clean    &66.2  &55.3    &57.6 &48.7  & 54.5  \\
&T-AP     & 11.5    &48.9 & 50.1  & 44.0    & 48.5\\
&TUE-MoCo   & 9.8   &46.2 &48.4   & 43.1    & 46.9\\
&AUE (ours)     &\textbf{8.7 }    & \textbf{15.0} & \textbf{20.4}  & \textbf{14.5}     & \textbf{18.2}  \\
&T-AAP (ours)   &29.8   &43.8   & 41.9  & 40.2 & 41.8   \\
\midrule
\multirow{3}{*}{I-100}& Clean  & 77.8 & 61.8   & 61.8 & 62.2 & 65.8 \\
 &AUE (ours)    & \textbf{5.1}  & \textbf{5.2}  & \textbf{6.2}& 
 \textbf{7.5} & 
 \textbf{4.7}  \\
 &T-AAP (ours)  & 14.4 & 20.3       & 14.5  & 24.8 & 16.6    \\
\bottomrule
\end{tabular}
}
\end{table}

\subsection{Efficiency of Poisoning Generation}
\label{subsec:efficiency}
In real-world scenarios, availability attacks need to generate perturbations for accumulating data as quickly as possible.
For expanding datasets, like continually updated social media user data, the poisoning used for data protection also needs to be updated periodically.
Since contrastive learning involves larger batches (i.e. at least 512) and a longer training process (i.e. 1000 epochs), these contrastive error minimization-based attacks require more time and memory consumption to generate perturbations.

In Table \ref{tab:efficiency}, we report the time cost of poisoning CIFAR-10/100 using a single NVIDIA 3090 GPU. 
Our supervised learning-based approaches are 3x, 6x, and 17x faster than TUE, CP, and TP.
Additionally, our methods admit smaller batches and simpler cross-entropy loss which require less memory, allowing for the generation of availability attacks on larger datasets with fewer devices.
Refer to Appendix \ref{app-subsec:computation-consumption} for more results about the efficiency of our methods.
\begin{table}[t]
\centering
\caption{Comparison in time consumption of poisoning algorithms for CIFAR-10/100 on the same device. 
Baseline methods adopt their default configurations.}
\label{tab:efficiency}
\resizebox{\columnwidth}{!}{
\begin{tabular}{c|ccc|cc}
\toprule
             & CP-SimCLR & TUE-MoCo  & TP & AUE & AAP \\
\midrule
Time cost & 48 hrs & 8.5 hrs & 16 hrs & 2.7 hrs & 2.2 hrs\\
\bottomrule
\end{tabular}
}
\end{table}

\subsection{More Evaluation Paradigms}
In addition to supervised learning and contrastive learning algorithms, we consider two more CL-like algorithms including supervised contrastive learning (SupCL, \citealt{khosla2020supervised}) and a semi-supervised learning algorithm FixMatch \cite{sohn2020fixmatch}. 
FixMatch uses WideResNet \cite{zagoruyko2016wide} and detailed settings are in Appendix \ref{app-subsec:evaluation-algorithms}.
Table \ref{tab:supcl-fixmatch} demonstrates that our attacks are still effective against SupCL and FixMatch. 
It indicates that our methods can handle more variants derived from supervised learning and contrastive learning algorithms.
\begin{table}[htb]
\centering
\caption{Attack performance (\%) of AUE and T-AAP against SupCL and FixMatch on CIFAR-10/100.}
\label{tab:supcl-fixmatch}
\resizebox{0.9\columnwidth}{!}{
\begin{tabular}{c|cc|cc}
\toprule
      & \multicolumn{2}{c|}{CIFAR-10} & \multicolumn{2}{c}{CIFAR-100} \\
      & SupCL       & FixMatch       & SupCL        & FixMatch       \\
\midrule
Clean &  94.6           & 95.7        &   72.5            &   77.0             \\ 
AUE   & 31.5         & 30.0           & \textbf{15.6}         & \textbf{12.0}           \\
T-AAP &\textbf{24.7}        & \textbf{18.7}           & 17.9        &   25.5     \\
\bottomrule
\end{tabular}
}
\end{table}

\subsection{Visualization}
We scale imperceptible perturbations from [-8/255, 8/255] to [0,1] and show their images in Figure \ref{fig:visualization-compare}. 
Enhanced data augmentations endow AUE with more complicated patterns than UE.
In terms of frequency, they are more high-frequency than UE.
Since contrastive augmentations include grayscale that squeezes low-frequency shortcuts \citep{liu2023image}, attacks against CL first need to come through them and thus prefer high-frequency patterns.
Moreover, we check the class-wise separability of perturbations using t-SNE visualization \cite{van2008visualizing} in Figure \ref{fig:visualization-compare}.
Perturbations from AUE and T-AAP are less separable than those from UE and T-AP and coincide with the characteristics of perturbations from contrastive error minimization \cite{he2022indiscriminate}.
Refer to Appendix \ref{app-subsec:visualization} for visualization of more attacks.

\begin{figure}[t]
    \centering
    \includegraphics[width=\columnwidth]{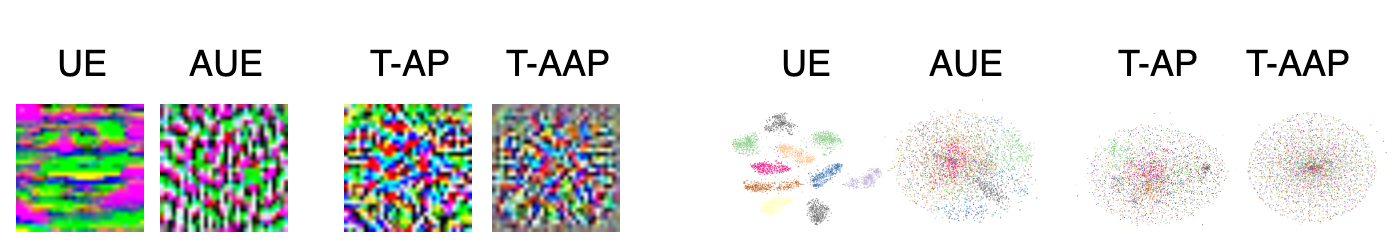}
    \caption{Visualization of poisoning on CIFAR-10. Left: Perturbation images. Right: T-SNE of perturbations.}
    \label{fig:visualization-compare}
\end{figure}

\subsection{Transferability across Networks}
Since the data protector is unaware of networks used in future training, availability attacks should be effective for different architectures.
We generate AUE and AAP using ResNet-18 and test them on ResNet-50, VGG-19 \citep{simonyan2015very}, DenseNet-121 \citep{huang2017densely}, and MobileNet v2 \citep{howard2017mobilenets, sandler2018mobilenetv2}.
In Table \ref{tab:network-archs}, both supervised unlearnability and contrastive unlearnability of AUE and AAP can transfer across these architectures. 
Moreover, the relative attack performance is preserved: T-AAP is consistently best for SL and UT-AAP is consistently best for CL.
% \vspace{-2pt}
\begin{table}[htb]
    \centering
    \caption{Transferability performance(\%) across architectures on CIFAR-10. %Contrastive learning uses the SimCLR algorithm.
    SimCLR is used for CL.}
\label{tab:network-archs}
\resizebox{0.7\columnwidth}{!}{
\begin{tabular}{ll|ccc}
\toprule
&Network            & AUE  & T-AAP & UT-AAP \\
\midrule
\multirow{4}{*}{SL}&ResNet-50    & 16.4 & 8.9   & 33.2   \\
&VGG-19       & 23.2 & 10.7  & 43.5   \\
&DenseNet-121 & 19.5 & 10.4  & 37.5   \\
&MobileNet v2   & 17.2 & 12.1  & 27.8   \\
\midrule
\multirow{4}{*}{CL}&ResNet-50    & 53.4 & 41.5  & 38.4   \\
&VGG-19       & 48.2 & 41.7  & 18.0   \\
&DenseNet-121 & 50.5 & 35.3  & 31.3   \\
&MobileNet v2   & 41.4 & 29.8  & 19.9   \\
\bottomrule
\end{tabular}
}
\end{table}
\vspace{-5pt}
\subsection{Training Process on Poisoned Data}\label{subsec:poisoned-training-process}
In Figure \ref{fig:early-stop}, we evaluate the training and test accuracy during SL and SimCLR training on poisoned data.
In very early epochs where the training underfits the poisoned data, checkpoints from both SL and SimCLR possibly process weak usability.
After a few epochs, the test accuracy rapidly goes down to an unusable level.
For SimCLR, the accuracy slowly increases in the middle and later stages of training.
It aligns with the overall trend of gradually decreasing uniformity gap and relatively stable alignment gap as shown in Figure \ref{subfig:gaps-of-AUE} for AUE.

\begin{figure}[htbp]
     \centering
%     \begin{subfigure}[b]{0.46\columnwidth}
%         \centering
         \includegraphics[width=0.23\textwidth]{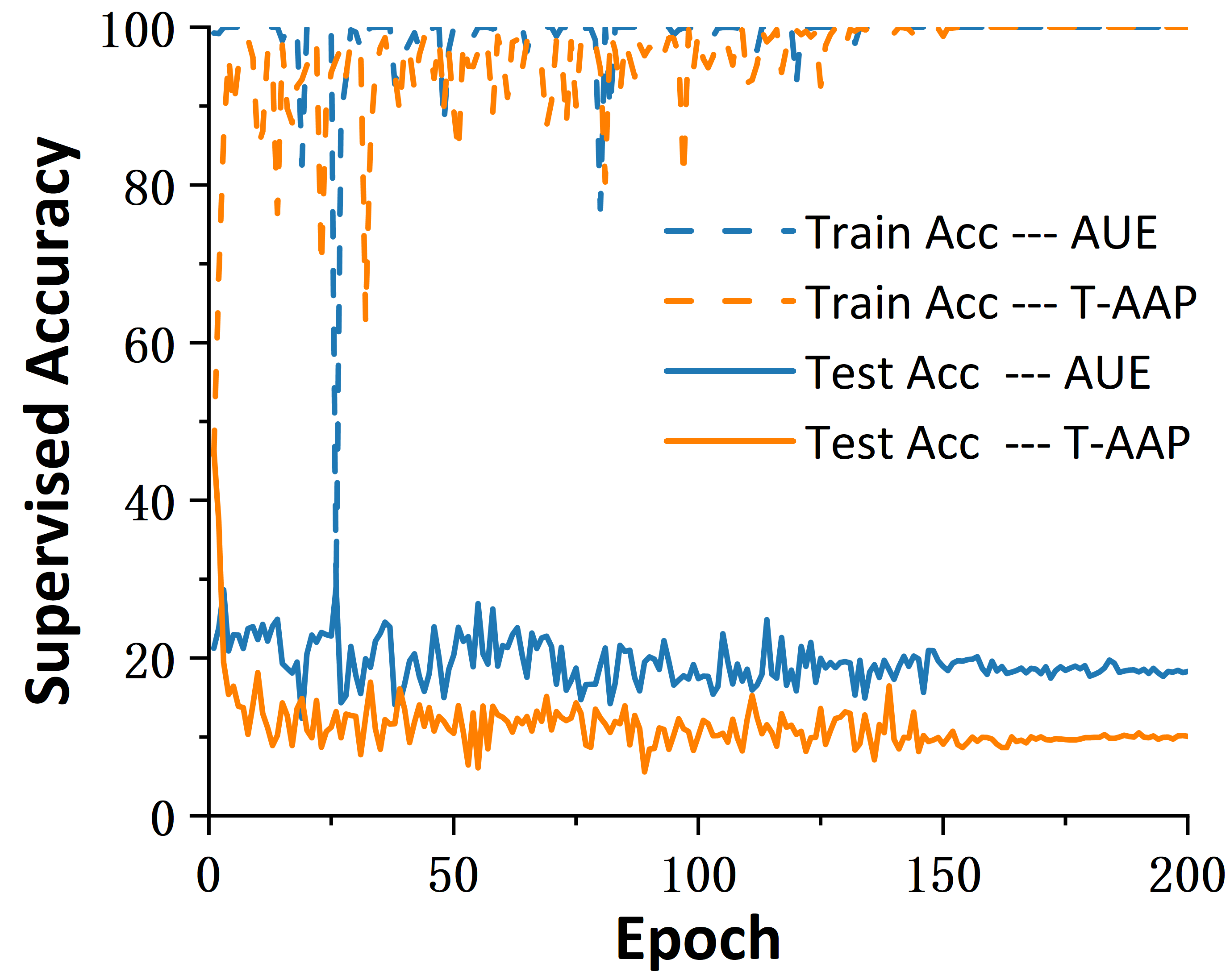}
%     \end{subfigure}
%     \begin{subfigure}[b]{0.46\columnwidth}
%         \centering
         \includegraphics[width=0.23\textwidth]{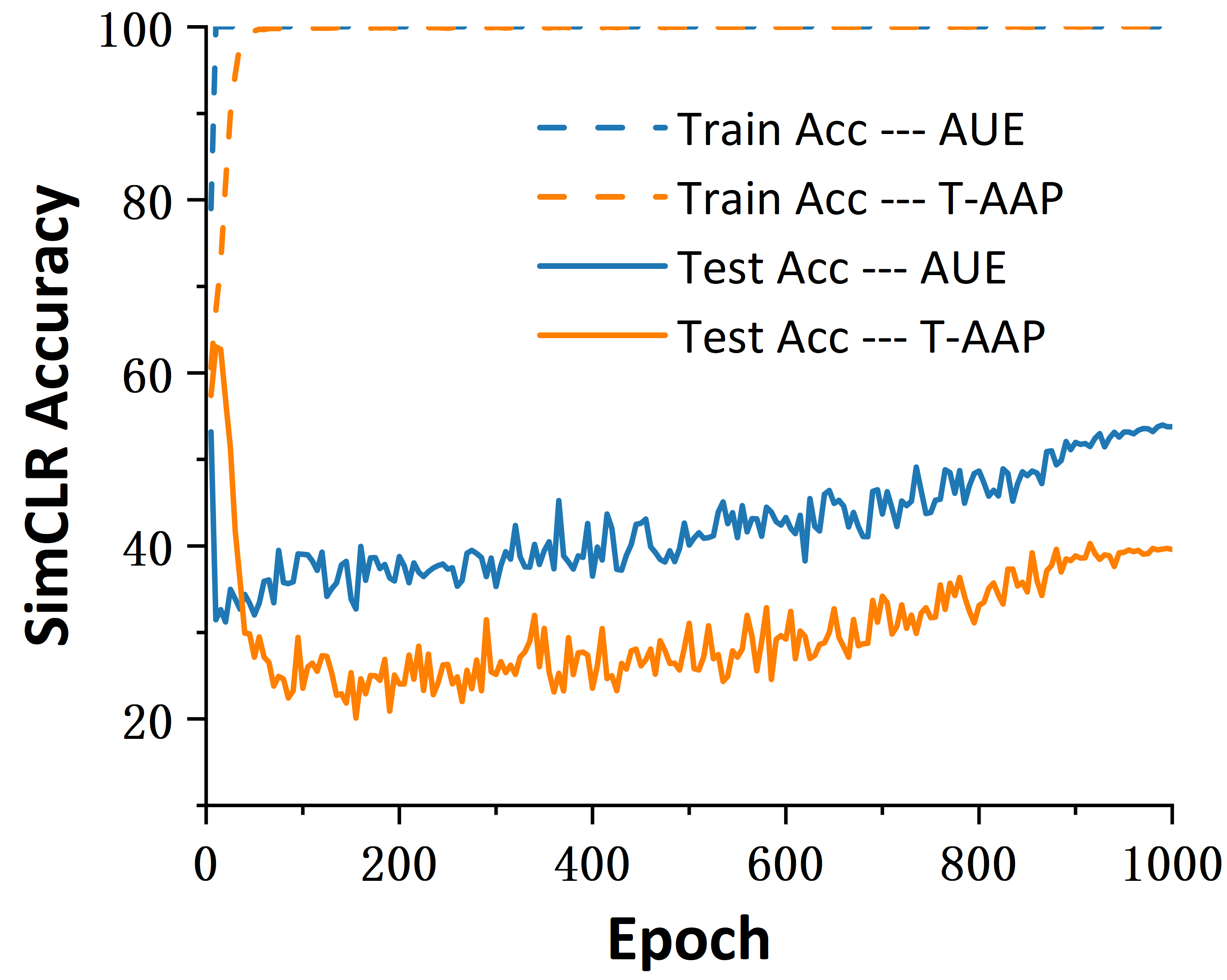}
%     \end{subfigure}
    \caption{Training process on poisoned CIFAR-10. Left: Supervised learning. Right: SimCLR.
    }  
    \label{fig:early-stop}
\end{figure}

\subsection{Ablation Study of Decoupling  Augmentations}
In settings of AUE and AAP, we control the strength of ResizedCrop, ColorJitter, and Grayscale through a single strength hyperparameter $s$ for the poison generation, as shown in Code \ref{code}. 
In Table \ref{tab:decoupling}, we decouple the strength hyperparameters for these three random transforms and evaluate the resulting attacks against SimCLR.
Different factors show different influences on the contrastive unlearnability for AUE and AAP.
For example, enhancing ResizedCrop strength alone is less effective than enhancing Grayscale alone in AUE generation.
However, adjusting three factors together generally outperforms other options in conclusion.

\begin{table}[ht]
\centering
\caption{SimCLR accuracy(\%) of attacks generated with decoupled strength parameters on CIFAR-10. For example, $0$-$0$-$s$ means that ResizedCrop strength is 0, ColorJitter strength is 0, and Grayscale strength is $s$.}
\label{tab:decoupling}
\resizebox{\columnwidth}{!}{
\begin{tabular}{ccccccccc}
\toprule
      & $0$-$0$-$0$ & $0$-$0$-$s$ & $0$-$s$-$0$ & $s$-$0$-$0$ & $0$-$s$-$s$ & $s$-$0$-$s$ & $s$-$s$-$0$ & $s$-$s$-$s$ \\
\midrule
AUE   & 83.5  & 58.7  & 79.4  & 88.7  & 60.8  & 56.2  & 87.7  & \textbf{52.4}  \\
T-AAP & 52.3  & 52.0  & 52.9  & 44.9  & 51.4  & 42.2  & 44.8  & \textbf{39.1}  \\
\bottomrule
\end{tabular}
}
\end{table}

\section{Conclusion}
Since contrastive learning algorithms bring new challenges to protect data using availability attacks, we explore attacks that have promising worst-case unlearnability.
While recently proposed methods focus on boosting the supervised unlearnability of contrastive error minimization, we propose to improve the contrastive unlearnability of supervised error minimization and maximization.
We introduce a very effective modification of data augmentation in supervised poisoning frameworks.
Our methods demonstrate superior performance and efficiency compared to existing methods, offering more potential in real-world applications.
%\input{files/impact}

%\bibliography{files/ref}
%\bibliographystyle{icml2024}

\onecolumn
\appendix

\section{Additional Related  Works}
\label{app-sec:additional-realated-works}
Availability attacks for supervised learning include 
error-minimizing noises \citep{huang2020unlearnable}, adversarial example poisoning \citep{fowl2021adversarial, chen2023self}, 
neural tangent generalization attack \citep{yuan2021neural}, generative poisoning attack \citep{feng2019learning}, autoregression perturbation \citep{sandoval2022autoregressive}, one-pixel perturbation \citep{wu2022one}, convolution-based attack \cite{sadasivan2023cuda}, synthetic perturbation \citep{yu2022availability},
and game-theoretic unlearnable examples \citep{Liu2024game}.  
\citet{yu2022availability} illustrated linearly separable perturbations work as shortcuts for supervised learning.
Robust error-minimizing noises \citep{Fu2022RobustUE}, entangled features strategy \citep{wen2023is}, and hypocritical perturbation \citep{tao2022can} were designed to deceive adversarial training.
Contrastive poisoning \citep{he2022indiscriminate} aimed at poisoning contrastive learning.
Transferable unlearnable examples \citep{ren2022transferable} and transferable poisoning \cite{liu2023transferable} improved the supervised unlearnability of contrastive poisoning.
%Unlearnable clusters 
\citet{zhang2023unlearnable} proposed to generate label-agnostic noises with cluster-wise perturbations.
On the defense side, adversarial training can largely mitigate the unlearnability \citep{tao2021better}. 
\citet{liu2023image, qin2023learning, zhu2023detection} leverages crafted data augmentations as defense.
\citet{sandoval2023can} suggests that the orthogonal projection technique is effective against class-wise attacks.
Diffusion models have been proposed to purify unlearnable perturbations \citep{jiang2023unlearnable, dolatabadi2023devil}.
\citet{qin2023apbench} introduced a benchmark for availability attacks.

When generating availability attacks, the gradient of perturbations is often computed through data augmentations. 
In literature, SL-based attacks generally use mild supervised data augmentation, i.e. RandomCrop and RandomHorizontalFlip \citep{fowl2021adversarial}.
The expectation over transformation (EOT) technique adopted by \citet{Fu2022RobustUE} first samples several such mild augmentations and then computes the average gradient over them.
CL-based attacks use contrastive augmentations \citep{he2022indiscriminate, ren2022transferable, liu2023transferable}.
To our knowledge, we are the first to use contrastive-like strong data augmentations in SL-based poisoning frameworks.

\section{Experiment Details}\label{app-sec:experiment-details}

\subsection{Datasets and Networks}\label{app-subsec:data}
\textbf{CIFAR.} CIFAR-10/CIFAR-100 \cite{krizhevsky2009learning} 
consists of 50000 training images and 10000 test images in 10/100 classes. 
All images are $32\times 32$ colored ones.

\textbf{Tiny-ImageNet.} Tiny-ImageNet classification challenge \citep{le2015tiny} is similar to the classification challenge in the full ImageNet ILSVRC \citep{russakovsky2015imagenet}. 
It contains 200 classes. The training has 500 images for each class and the test set has 100 images for each class. 
All images are $64\times 64$ colored ones.

\textbf{Mini-ImageNet.} Mini-ImageNet dataset was originally designed for few-shot learning \citep{vinyals2016matching}. We modify it for a classification task.
The modified dataset contains 100 classes. The training set has 500 images for each class. The test set has 100 images for each class.
All images are $84\times 84$ colored ones.

\textbf{ImageNet-100.} ImageNet-100 is a subset of ImageNet-1k Dataset from ImageNet Large Scale Visual Recognition Challenge 2012 \citep{russakovsky2015imagenet}. It contains 100 random classes. The training set has 130,000 images. The test set has 5,000 images. Images are processed to 224x224 colored ones as input data to models.

\textbf{ResNet.} 
% We use the ResNet models in Torchvision without any pretrained weights. 
On CIFAR-10/CIFAR-100, we set the kernel size of the first convolutional layer to 3 and removed the following max-pooling layer. On other datasets, we do not modify the models.

\subsection{Data Augmentation}\label{app-subsec:augmentations}
In Code \ref{code}, we show the different implementations of data augmentation between supervised learning and contrastive learning.
For supervised learning, we consider the typical augmentations including Crop and HorizontalFlip.
For contrastive learning, we consider the typical augmentations including ResizedCrop, HorizontalFlip, ColorJitter, and Grayscale, and its default strength $s=1$.
In the generation process of our AUE and AAP attacks, we replace the supervised augmentations with contrastive-like augmentations of a strength parameter $s$.
\newpage
\begin{lstlisting}[style=mypython,caption={Different data augmentations used in supervised learning and contrastive learning on CIFAR-10/100 datasets. The intensity of contrastive augmentations can be adjusted via strength $s$.}, label={code}]
# Supervised augmentations
Compose([RandomCrop(size=32, padding=4), RandomHorizontalFlip(p=0.5),
         ToTensor()])
# Contrastive augmentations
s = 1.0  # Strength is 1.0 by default for contrastive learning.
Compose([RandomResizedCrop(size=32, scale=(1-0.9*s, 1.0)),
         RandomHorizontalFlip(p=0.5),
         RandomApply([ColorJitter(brightness=0.4*s, contrast=0.4*s, 
                                saturation=0.4*s, hue=0.1*s)], p=0.8*s),
         RandomGrayscale(p=0.2*s), ToTensor()])
\end{lstlisting}

\subsection{Details of AUE and AAP}
\label{app-subsec:details-of-aue-aap}
We leverage differentiable augmentation modules in Konia\footnote{https://github.com/kornia/kornia} \citep{riba2020kornia} which is a differentiable computer vision library for PyTorch.
The contrastive augmentations for Tiny/Mini-ImageNet and ImageNet-100 are similar to those for CIFAR-10/100 in Code \ref{code} but only adapt the image size.

\textbf{AUE.} We train the reference model for $T=60$ epochs with SGD optimizer and cosine annealing learning rate scheduler. The batch size of training data is $128$.
The initial learning rate $\alpha_\theta$ is $0.1$, weight decay is $10^{-4}$ and momentum is $0.9$.
In each epoch, we update the model for $T_\theta = 391$ iterations and update poisons for $T_\delta = 391$ iterations. 
For ImageNet-100, we set  $T_\theta = T_\delta = 1016$.
The PGD process for noise generation takes $T_p=5$ steps with step size $\alpha_\delta = 0.8/255$.
The augmentation strength $s=0.6$ for CIFAR-10 and $s=1.0$ for CIFAR-100, Tiny-ImageNet, Mini-ImageNet, and ImageNet-100. 
Additional experiments of the selection of strength parameters are shown in Appendix \ref{app-subsec:strength-selection}.

\textbf{AAP.} We train the reference model for $T=40$ epochs, and the initial learning rate $\alpha_\theta$ is $0.5$. The PGD process for noise generation takes $T_p=250$ steps with step size $\alpha_\delta = 0.08/255$. Other settings are the same as AUE.
The label translation is $K=1$.
The augmentation strength $s=0.4$ for CIFAR-10 and $s=0.8$ for CIFAR-100, Tiny-ImageNet, Mini-ImageNet, and ImageNet-100.

\textbf{Sample-wise Attack.}
When a poisoning map $\delta(\boldsymbol{x},y)$ only depends on label $y$, the resulting attack is called a class-wise attack; otherwise, it is a sample-wise attack.
In this paper, we focus on sample-wise attacks.

\subsection{Evaluation Algorithms}
\label{app-subsec:evaluation-algorithms}
\textbf{Contrastive learning.}
The setup for SimCLR, MoCo, BYOL, and SimSiam are shown in Table \ref{tab:hyper-params}.
The 100-epoch linear probing stage uses an SGD optimizer and a scheduler that decays 0.2 at 60, 75, and 90 epochs.
The probing learning rate is 1.0 for SimCLR, MoCo, BYOL, and 5.0 for SimSiam on CIFAR-10/100, Tiny/Mini-ImageNet.
On ImageNet-100, the unsupervised contrastive learning optimizes 200 epochs and the linear probing uses a learning rate of 10.0. Other settings are the same as other datasets.
After generating our attacks on CIFAR-10/100, we report average test accuracy after 3 evaluations with random seeds.

\textbf{Supervised learning.} 
We augment the training data by RandomHorizontalFlip and RandomCrop with padding size $l/8$  on CIFAR-10/100 and Tiny/Mini-ImageNet. $l$ is the image size.
On ImageNet-100, we augment using RandomResizedCrop and RandomHorizontalFlip.

\textbf{SupCL and FixMatch.}
We use ResNet-18 for SupCL evaluation on CIFAR-10 and CIFAR-100.
For FixMatch evaluation, we use WideResNet-28-2 and 4000 labeled data on CIFAR-10; we use WideResNet-28-8 and 10000 labeled data on CIFAR-100.

\begin{table}[htb]
\centering
\caption{Details of supervised and contrastive evaluations.}
\label{tab:hyper-params}
\begin{tabular}{llllll}
\toprule[1pt] %\specialrule{0em}{0.8pt}{0.8pt} 
                 & SL     & SimCLR  & MoCo     & BYOL   & SimSiam \\ %\specialrule{0em}{0.8pt}{0.8pt} 
\midrule  %\specialrule{0em}{0.8pt}{0.8pt} 
Batch size       & 512    & 512     & 512     & 512    & 512            \\%\specialrule{0em}{0.8pt}{0.8pt} 
Epochs           & 200    & 1000    & 1000    & 1000   & 1000            \\%\specialrule{0em}{0.8pt}{0.8pt} 
Loss function    & CE     & InfoNCE & InfoNCE & MSE    & Similarity             \\%\specialrule{0em}{0.8pt}{0.8pt} 
Optimizer        & SGD    & SGD     & SGD     & SGD    & SGD            \\%\specialrule{0em}{0.8pt}{0.8pt} 
Learning rate    & 0.5    & 0.5     & 0.3     & 1.0    & 0.1            \\%\specialrule{0em}{0.8pt}{0.8pt} 
Weight decay     & 1e-4   & 1e-4    & 1e-4    & 1e-4   & 1e-4              \\%\specialrule{0em}{0.8pt}{0.8pt} 
Momentum         & 0.9    & 0.9     & 0.9     & 0.9    & 0.9            \\%\specialrule{0em}{0.8pt}{0.8pt} 
Scheduler        & Cosine & Cosine  & Cosine  & Cosine & Cosine    \\%\specialrule{0em}{0.8pt}{0.8pt} 
Warmup           & 10     & 10      & 10      & 10     & 10              \\%\specialrule{0em}{0.8pt}{0.8pt} 
Temperature      & -      & 0.5     & 0.2     & -      & -              \\%\specialrule{0em}{0.8pt}{0.8pt} 
Encoder momentum & -      & -       & 0.99    & 0.999  & -              \\%\specialrule{0em}{0.8pt}{0.8pt} 
\bottomrule[1pt]
\end{tabular}
\end{table}

\section{Additional Experiments} \label{app-sec:additional-experiments}

\subsection{Computation Consumption}
\label{app-subsec:computation-consumption}
We report the time consumption of generating AUE and AAP attacks.
For CIFAR-10/100, Tiny/Mini-ImageNet, experiments are conducted using a single NVIDIA GeForce RTX 3090 GPU.
For ImageNet-100, experiments are conducted using a single NVIDIA A800 GPU.
On CIFAR-10/100, AUE/AAP costs around 2.7/2.2 hours.
On Mini-ImageNet, AUE/AAP costs around 2.5/2 hours.
On Tiny-ImageNet, AUE/AAP costs around 2.5/3.8 hours.
On ImageNet-100, AUE/AAP costs around 12/10 hours.
In comparison, on CIFAR-10/100 and using the same device, CP-SimCLR costs around 48 hours, TUE-MoCo costs around 8.5 hours, and TP costs around 16 hours to generate poisons.
Our supervised poisoning attacks are much more efficient than contrastive poisoning attacks.

\subsection{Visualization}
\label{app-subsec:visualization}
In Figure \ref{fig:full-images-tsne}, we present images and the t-SNEs visualization of availability attacks on CIFAR-10. 
\begin{figure}[htbp]
     \centering
     \begin{subfigure}[b]{0.6\columnwidth}
         \centering
         \includegraphics[width=\textwidth]{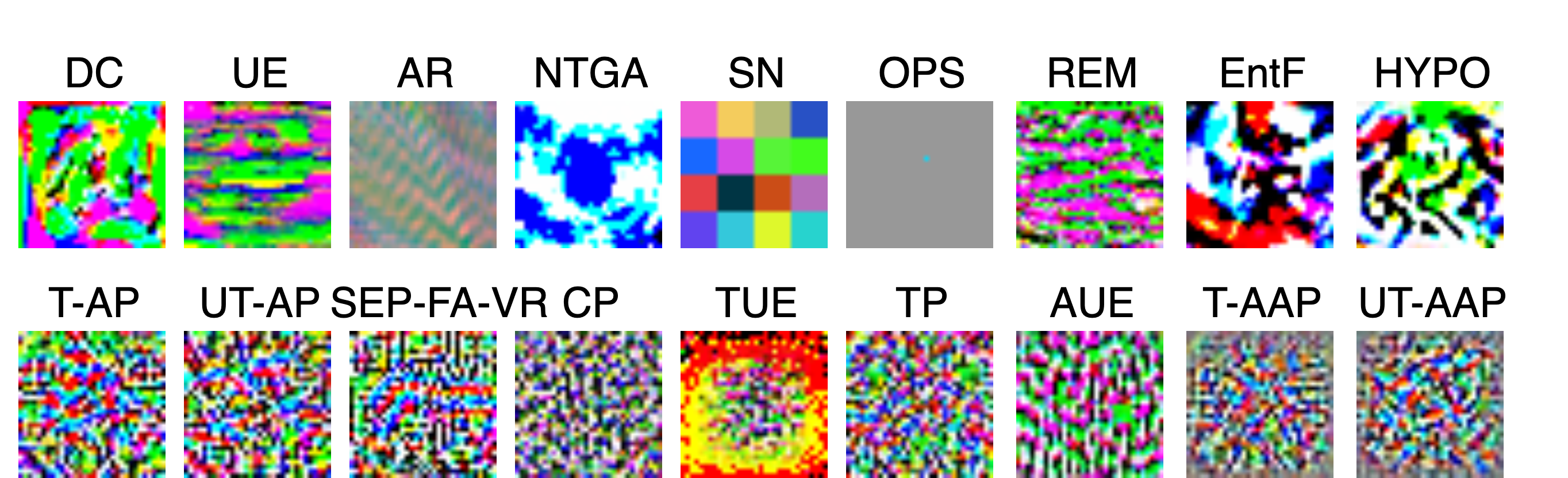}
         \caption{}
         \label{fig:poison-imgs}
     \end{subfigure}
     \begin{subfigure}[b]{0.6\columnwidth}
         \centering
         \includegraphics[width=\textwidth]{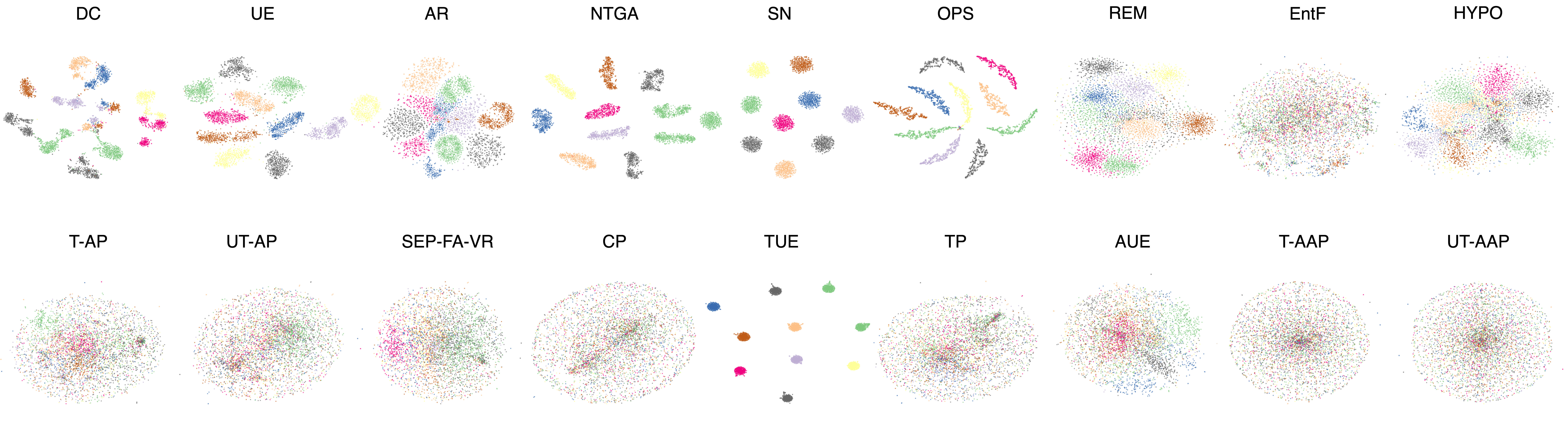}
         \caption{}
         \label{fig:tsne-pairs}
     \end{subfigure}
    \caption{(a) Perturbation images of availability attacks on CIFAR-10.
    (b) T-SNE visualization of perturbations.    
    In each figure, the top row includes attacks that are not effective against contrastive learning, and the bottom row includes attacks that have contrastive unlearnability.
    }  
    \label{fig:full-images-tsne}
\end{figure}

\subsection{Strength Selection}
\label{app-subsec:strength-selection}
\begin{figure}[htbp]
     \centering
     \begin{subfigure}[b]{0.25\textwidth}
         \centering
         \includegraphics[width=\textwidth]{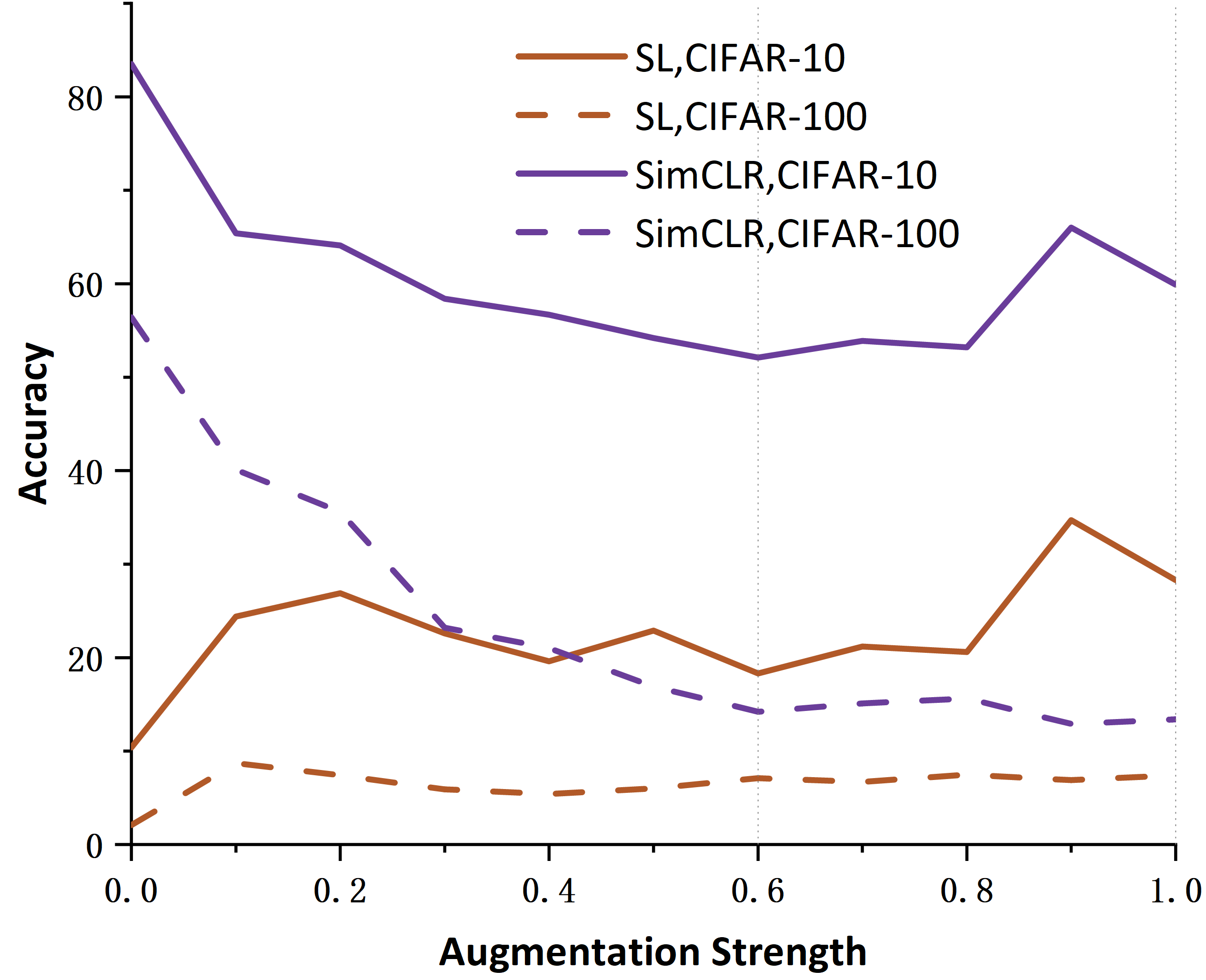}
         \caption{}
         \label{subfig:UE-with-enhanced-data-augmentations}
     \end{subfigure}
     \begin{subfigure}[b]{0.25\textwidth}
         \centering
         \includegraphics[width=\textwidth]{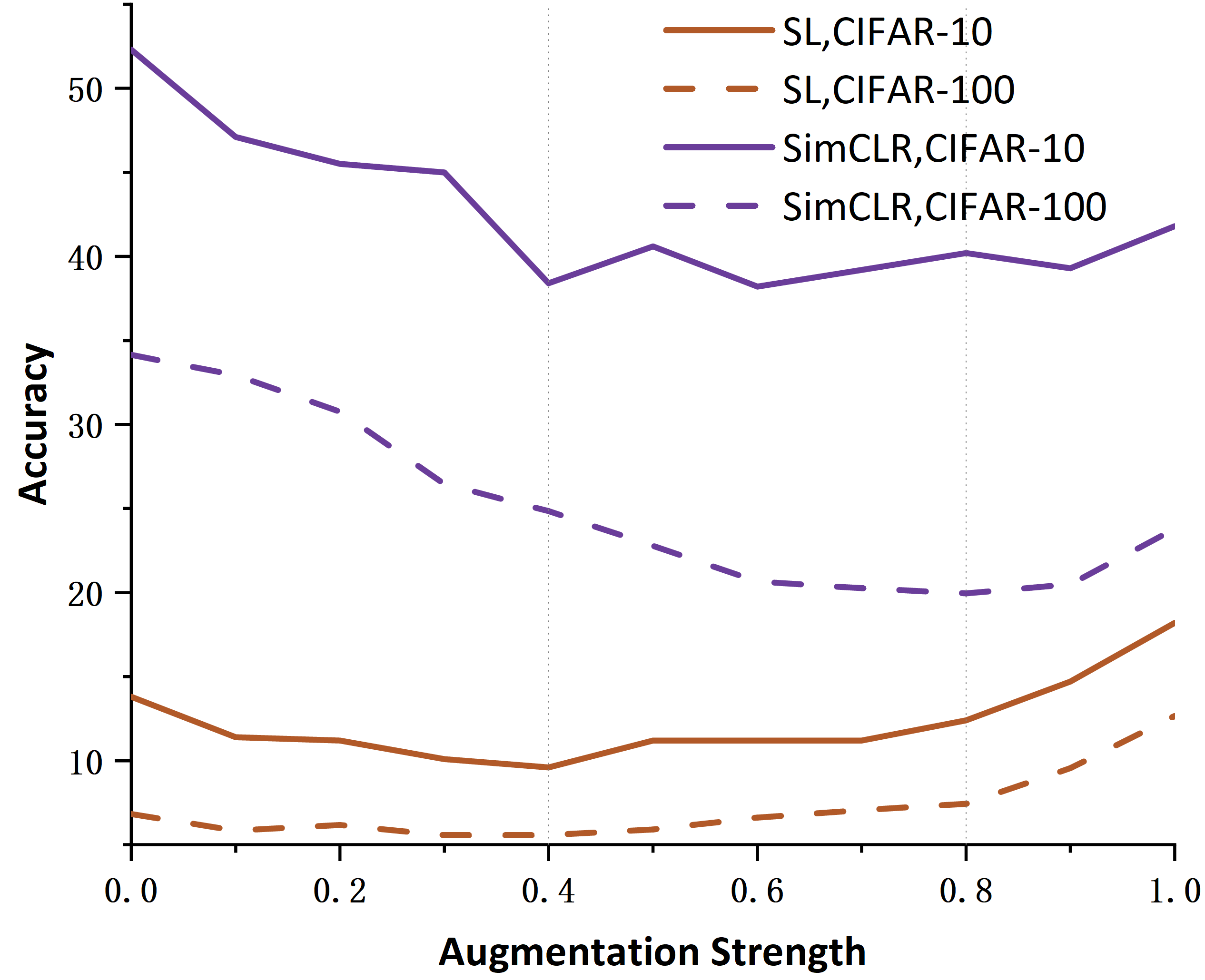}
         \caption{}
         \label{subfig:TAP-with-enhanced-data-augmentations}
     \end{subfigure}
    \caption{
    (a) Influence of augmentations in AUE.  
    (b) Influence of augmentations in T-AAP.
    }
\end{figure}
\textbf{AUE.}
We gradually increase the data augmentation strength $s$ in the supervised error minimization according to Algorithm \ref{alg:aue}.
In Figure \ref{subfig:UE-with-enhanced-data-augmentations}, the SimCLR accuracy prominently decreases as the strength grows, while the supervised learning accuracy slightly increases.
Compared to UE, our AUE attacks largely improve contrastive unlearnability while keeping similar supervised unlearnability.
On CIFAR-10, too strong strengths might compromise the unlearnability. 
% make poison generation hard to converge and 
Thus, we generate our augmented unlearnable example (AUE) attacks taking $s=0.6$ for CIFAR-10, and $s=1.0$ for CIFAR-100.

\textbf{AAP.}
We gradually increase the data augmentation strength $s$ in the supervised error maximization according to Algorithm \ref{alg:aap}.
In Figure \ref{subfig:TAP-with-enhanced-data-augmentations}, the SimCLR accuracy decreases with the strength, while the supervised learning accuracy slightly increases.
Proper augmentation strengths improve the contrastive unlearnability but too large $s$ might introduce difficulty in poison generation and harm the supervised unlearnability.
We select $s=0.4$ for CIFAR-10 and $s=0.8$ for CIFAR-100.

\subsection{Strength and Gaps}
On CIFAR-10, we gradually increase the augmentation strength from 0 to the default setting, i.e. $s=0.6$ in the generation of AUE attacks and evaluate the alignment gaps, uniformity gaps, and the SimCLR Accuracy in Table \ref{tab:strength-gaps-aue}.
In this case, the larger the gaps, the lower the accuracy of SimCLR. 

\begin{table}[h]
\centering
\caption{Alignment and uniformity gaps caused by AUE attacks with different data augmentation strengths on CIFAR-10.}
\label{tab:strength-gaps-aue}
\begin{tabular}{cccc}
\toprule
Strength & \multicolumn{1}{l}{Alignment Gap} & Uniformity Gap & \multicolumn{1}{l}{SimCLR Accuracy} \\
\midrule
$s=0.0$        & 0.14                              & 0.07           & 83.5                                \\
$s=0.2$      & 0.21                              & 0.24           & 64.1                                \\
$s=0.4$      & 0.25                              & 0.28           & 56.7                                \\
$s=0.6$      & 0.27                              & 0.34           & 52.4                                \\
\bottomrule
\end{tabular}
\end{table}

\subsection{Poisoning Budget} 
In the main body, we consider the poisoning attacks constrained in a $L_\infty$-norm ball with radius $8/255$. 
The constraint is to ensure perturbations are imperceptible to human eyes.
We investigate the influence of different poisoning budgets. 
AUE and AAP attacks are generated with poisoning budgets of $2/255, 4/255, 6/255$ and are evaluated by SL and SimCLR.
In Table \ref{tab:poisoning-budget}, the larger the poisoning budgets, the better the attack performance.
\begin{table}[htb]
    \centering
    \caption{Performance(\%) of attacks generated with different poisoning budgets on CIFAR-10.}
\label{tab:poisoning-budget}
    \begin{tabular}{ccccc}
\toprule
& Budget            & AUE  & T-AAP & UT-AAP \\
\midrule
\multirow{3}{*}{SL}&2/255     & 34.5 & 50.7  & 75.6   \\
&4/255     & 28.5 & 19.7  & 58.5   \\
&6/255     & 26.8 & 12.3  & 44.2   \\
\midrule
\multirow{3}{*}{SimCLR}&2/255 & 84.8 & 87.0  & 87.1   \\
&4/255 & 70.1 & 66.6  & 59.8   \\
&6/255 & 59.4 & 51.1  & 43.0   \\
\bottomrule
\end{tabular}
\end{table}

\subsection{Defense}
\label{subsec:defenses}
On the defense side against availability attacks, AT \citep{madry2018towards} and AdvCL \citep{kim2020adversarial}) applied adversarial training in supervised learning and contrastive learning respectively;
ISS \citep{liu2023image} and UEraser \citep{qin2023learning} leveraged designed data augmentations to eliminate supervised unlearnability; 
AVATAR \citep{dolatabadi2023devil} employed a diffusion model to purify poisoned data.
In Table \ref{tab:defenses}, we evaluate our attacks through these defense methods as well as SimCLR with Cutout \citep{devries2017improved}, Random noise, and Gaussian Blur.
The defensive budget for AT and AdvCL is $8/255$; the length parameter for Couout is 8; the kernel size for Gaussian Blur is 3; the variance for Random noise is 8/255.

The defense performance of a method differs when facing different attacks.
For example, UEraser can recover the accuracy of TUE-SimCLR from $10.6\%$ to above $92.7\%$, while its effect on our AUE attack is much weaker.
At the cost of a significant amount of extra training time, adversarial training, i.e. AT and AdvCL, can increase accuracy to around $80\%$.
ISS mitigates the supervised unlearnability of evaluated attacks back to levels close to $85\%$, but its Grayscale component may even have negative effects.
Gaussian Blur is more effective than Cutout and Random noise for contrastive learning.
AVATAR seems to achieve the best defense performance against our proposed attacks, but the final accuracy still exhibits a gap compared to training with clean data, such as achieving accuracy above $90\%$.

\begin{table}[htbp]
\centering
\caption{Performance(\%) under defenses on CIFAR-10.}
\label{tab:defenses}
% \resizebox{0.4\columnwidth}{!}{
\begin{tabular}{cl|ccc|cc}
\toprule
 & Defense              & AUE  & T-AAP & UT-AAP &T-AP&TUE-SimCLR\\
\midrule
\multirow{9}{*}{SL}&No Defense & 18.9  &  9.2 & 29.7 &9.5&10.6 \\
&UEraser         & 63.2   & 64.7  & 81.8  &68.0&92.7 \\
&\ \ \ -Lite    & 60.6     &66.8  &82.2 &70.7&92.2\\
&\ \ \ -Max     & 72.8     &79.5  &85.8   &80.2&93.2\\
&ISS             & 82.6     &82.3  &81.4   &81.7&82.7\\
&\ \ \ -Grayscale   & 18.2     &9.1  &23.8  &11.4&28.0 \\
&\ \ \ -JPEG        & 84.9     &84.3  &84.0  &84.6&82.1 \\
&AVATAR          & 85.0     &88.0  &86.6  &87.7&83.2 \\
&AT             & 83.8     & 81.6  & 79.6  &81.0&81.7 \\
\midrule
\multirow{6}{*}{SimCLR}&No Defense &52.4  &  39.1 &  32.3 & 48.4&48.1 \\
&Cutout         & 51.8 & 37.9  & 31.8 &49.2&49.6\\
&Random Noise    & 60.5 & 62.4  & 48.0 &66.4&70.0  \\
&Gaussian Blur& 69.1 & 76.7  & 78.9 &75.5&79.3  \\
&AVATAR          &83.1&80.8&79.9&81.1& 83.0\\
&AdvCL          & 80.9 & 78.4  & 77.5  &78.8&80.1 \\
\toprule
\end{tabular}
% }
\end{table}

\subsection{Poisoning Ratio}
Availability attacks are sensitive to the proportion of poisoned data in the dataset and usually need to poison the whole dataset \cite{huang2020unlearnable, fowl2021adversarial}.
In the main body, we report results when the poisoning ratio is $100\%$.
Here, we investigate the influence of the poisoning ratio on the attack performance of AUE and AAP.
Table \ref{tab:poisoning-ratio} illustrates that our augmented methods inherit the vulnerability to poisoning ratio from basic approaches, i.e. UE and AP, though AUE is more robust than AAP.
This characteristic also necessitates the prompt processing of newly acquired clean data, imposing higher efficiency demands on the generation of attacks.
\begin{table}[h]
\centering
\caption{Performance(\%) of attacks with different poisoning ratios on CIFAR-10.}
\label{tab:poisoning-ratio}
\begin{tabular}{ccccc}
\toprule
&Ratio & AUE  & T-AAP & UT-AAP \\
\midrule
\multirow{3}{*}{SL}&$95\%$         & 75.6 & 82.1  & 84.2   \\
&$90\%$          & 82.2 & 86.6  & 89.2   \\
&$80\%$          & 87.6 & 89.8  & 91.2   \\
\midrule
\multirow{3}{*}{SimCLR}&$95\%$         & 69.7 & 76.8  & 74.2   \\
&$90\%$         & 74.5 & 82.1  & 79.9   \\
&$80\%$         & 79.7 & 85.5  & 83.9   \\
\bottomrule
\end{tabular}
\end{table}

\subsection{Discussion of Clean Linear Probing}\label{app-subsect: discussion-of-clean-linear-probing}
While our threat model linear probes on poisoned data,
\citet{he2022indiscriminate} use clean data for linear probing instead. 
In Figure \ref{fig:linear-probe}, we compare the final classification performance of SimCLR models in these two settings.
Feature extractors are trained on poisoned data and are fixed.
We focus on the classification performance after linear probing on clean or poisoned data.
While CP and TUE obtain similar attack performance in both cases, clean linear probing can mitigate SL-based attacks including AP, SEP-FA-VR, AAP, and AUE.
On one hand, for SL-based poisoning, the dissimilarities between clean features and poisoned features hinder a classifier head obtained by poisoned linear probing in generalizing to clean features, as discussed in Section \ref{subsec:existing-attacks-against-cl}.
However, clean features still contain useful information and can derive another classifier head to perform classification.
On the other hand, contrastive error-minimizing noises confuse the feature extractor directly such that even clean data fail to activate useful features for classification.
But in general, given a responsible data publisher who protects data using availability attacks before release, an unauthorized data collector has no access to unprocessed data for clean linear probing. 
Thus, it is sufficient to achieve contrastive unlearnability with poisoned linear probing in real scenarios.

\begin{figure}[t]
    \centering
    \includegraphics[width=0.35\columnwidth]{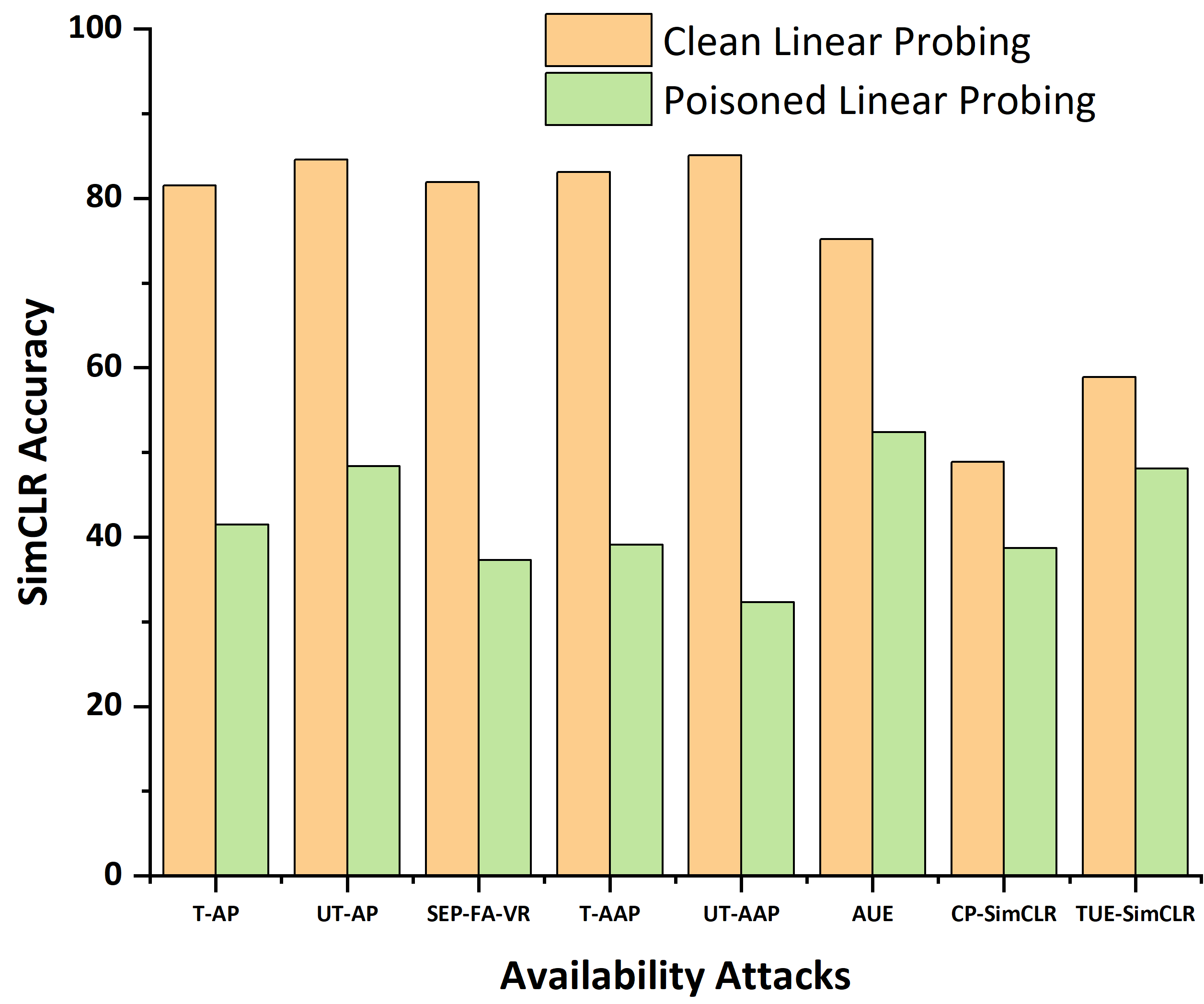}
    \caption{Clean and poisoned linear probing on CIFAR-10.}
    \label{fig:linear-probe}
\end{figure}

\section{Toy Example}
\label{app-sec:proofs}
We study a model $f=h\circ g:\R^d\to \R^n$ with a normalized feature extractor $g:\R^d \to \R^n$ such that $||g(\boldsymbol{x})||\equiv 1$ and a full rank linear classifier $h:\R^n\to \R^n$ in the sense that $h(\boldsymbol{z})=W \boldsymbol{z} + \boldsymbol{b}$ with a full rank square matrix $W\in \R^{n\times n}$.
By singular values decomposition (SVD), $W=U\Sigma V$ with orthogonal matrices $U, V\in \R^{n\times n}$ and $\Sigma = \text{diag}(\sigma_1,\cdots,\sigma_n), \sigma_1 \geq \cdots \geq \sigma_n >0$. 
Let $\Dd$ be a balanced data distribution, i.e. each class would be sampled with the same probability, $\Dd_x$ be the margin distribution, and $\mu$ be an augmentation distribution. 
Assume the supervised loss $\L_{\SL}$ is the mean squared error, and the contrastive loss $\L_{\CL}$ contains only one negative example:
\begin{align*}
\L_{\SL}(\boldsymbol{x},y,\pi) &= \frac{1}{n}|| h\circ g (\pi(\boldsymbol{x})) - \boldsymbol{e}_y||^2\\
\L_{\CL}(\boldsymbol{x}, \boldsymbol{x}^-, \pi, \tau, \rho)
&=  \log(1+ \frac{e^{g(\pi(\boldsymbol{x}))^\top g(\rho(\boldsymbol{x}))}}{e^{g(\pi(\boldsymbol{x}))^\top g(\tau(\boldsymbol{x}^-))}}).
\end{align*}

\begin{proposition} \label{thm:main}
Let $\Ee_{\SL}
     =\E_{\Dd, \mu} 
     \big[\L_{\SL}(\boldsymbol{x},y,\pi)\big]$.
With probability at least $1-4\sqrt{\Ee_{\SL}}$, it holds 
\begin{small}
 \begin{align*}
    &\L_{\CL}(\boldsymbol{x}, \boldsymbol{x}^-, \pi, \tau, \rho)<\frac{1}{n} \log (1+\frac{\sigma_n}{\sigma_n-2n\sqrt{\Ee_{\SL}}})\\
    &+ \frac{n-1}{n} \log (1+\frac{\sigma_1^2\sigma_n-\sigma_n(1-\sqrt{2n\sqrt{\Ee_{\SL}}})^2}{\sigma_1^2\sigma_n-2n\sigma_1^2\sqrt{\Ee_{\SL}}}).
    \end{align*}
\end{small}
\end{proposition}
\begin{remark}
1) Assumptions of a square matrix and positive singular values are necessary. Otherwise, the dimensional reduction of feature space impairs the relation between supervised and contrastive losses. 
2) Since supervised losses contain limited information about negative pairs, this inequality is naturally loose. 
However, in the case that supervised learning fits very well, it at least implies that positive features $g(\tau(\boldsymbol{x}))$ are closer to $g(\pi(\boldsymbol{x}))$ than negative features $g(\rho(\boldsymbol{x^-}))$.
\end{remark}

\subsection{Lemmas}
% We use notations in Proposition \ref{thm:main}.

\begin{lemma} \label{lem:sigular-value-bound}
For any $\boldsymbol{z}\in \R^n$,
\begin{align*}
    \sigma_n ||\boldsymbol{z}|| \leq ||W \boldsymbol{z}|| \leq \sigma_1 ||\boldsymbol{z}||.
\end{align*}
\end{lemma}
\begin{proof}
Denote $\Tilde{\boldsymbol{z}} = (\Tilde{z}_1, \cdots,\Tilde{z}_n)^\top = V\boldsymbol{z}$. 
Since orthogonal matrices preserve the norm,
\begin{align*}
    ||W \boldsymbol{z}|| = ||U\Sigma V \boldsymbol{z}||
    = ||\Sigma \Tilde{\boldsymbol{z}}||
    = \sqrt{\sum_{i=1}^n \sigma_i^2 \Tilde{z_i}^2 },
\end{align*}
\begin{align*}
    \sigma_n ||\boldsymbol{z}|| = \sigma_n||\Tilde{\boldsymbol{z}}|| 
    \leq \sqrt{\sum_{i=1}^n \sigma_i^2 \Tilde{z_i}^2 } 
    \leq \sigma_1||\Tilde{\boldsymbol{z}}|| = \sigma_1 ||\boldsymbol{z}||.
\end{align*}
\end{proof}

\begin{lemma} \label{lem:markov-inequ}
If $\Ee_{\SL}\leq \epsilon$, then with probability at least $1-\sqrt{\epsilon}$
\begin{align*}
    ||h\circ g(\pi(\boldsymbol{x})) - \boldsymbol{e}_y|| < \sqrt{n \sqrt{\epsilon}},
\end{align*}
where $(\boldsymbol{x},y)\sim \Dd, \pi\sim \mu$.
\end{lemma}
\begin{proof}
    As 
    \begin{align*}
    \Ee_{\SL} =  \E_{ \substack{(\boldsymbol{x},y)\sim \Dd \\ \pi \sim \mu}} \big[\frac{1}{n}|| h\circ g (\pi(\boldsymbol{x})) - \boldsymbol{e}_y||^2\big],
    \end{align*}
    by Markov's inequality, it has
    \begin{align*}
        \text{Pr}(\frac{1}{n}|| h\circ g (\pi(\boldsymbol{x})) - \boldsymbol{e}_y||^2 \geq \sqrt{\epsilon})\leq \sqrt{\epsilon}.
    \end{align*}
\end{proof}

\begin{lemma} \label{lem:same-x}
If $\Ee_{\SL}\leq \epsilon$, then with probability at least $1-2\sqrt{\epsilon}$
\begin{align*}
    g(\pi(\boldsymbol{x}))^\top g(\tau(\boldsymbol{x})) >  1 - \frac{2n\sqrt{\epsilon}}{\sigma_n},
\end{align*}
where $\boldsymbol{x}\sim \Dd_{\boldsymbol{x}}, \pi,\tau \sim \mu$.
\end{lemma}
\begin{proof}
    By Lemma \ref{lem:markov-inequ},
    with probability at least $1-2\sqrt{\epsilon}$,
    \begin{align*}
        ||h\circ g(\pi(\boldsymbol{x})) - \boldsymbol{e}_y|| < \sqrt{n \sqrt{\epsilon}} 
        \text{\ \ \ \ and\ \ \ \ }
        ||h\circ g(\tau(\boldsymbol{x})) - \boldsymbol{e}_y|| < \sqrt{n \sqrt{\epsilon}}.
    \end{align*}
    By the triangle inequality,
    \begin{align*}
        ||h\circ g(\pi(\boldsymbol{x})) - h\circ g(\tau(\boldsymbol{x}))|| < 2 \sqrt{n\sqrt{\epsilon}}
    \end{align*}
    
    Since $g$ is normalized, by Lemma \ref{lem:sigular-value-bound} we have
    \begin{align*}
         g(\pi(\boldsymbol{x}))^\top g(\tau(\boldsymbol{x})) &= 1 - \frac{1}{2} || g(\pi(\boldsymbol{x}))- g(\tau(\boldsymbol{x}))||^2\\
         &\geq 1 - \frac{1}{2\sigma_n^2} ||h\circ g(\pi(\boldsymbol{x})) - h\circ g(\tau(\boldsymbol{x})||^2 \\ 
         &> 1 - \frac{2n\sqrt{\epsilon}}{\sigma_n}.
         \end{align*}    
\end{proof}

\begin{lemma} \label{lem:diff-x}
Assume $\Dd$ is a balanced dataset.
If $\Ee_{\SL}\leq \epsilon$, then with probability at least $1-2\sqrt{\epsilon}$, one of the following two conditions holds
\begin{enumerate}
    \item with probability $\frac{n-1}{n}$, 
    \begin{align*}
        g(\pi(\boldsymbol{x}))^\top g(\tau(\boldsymbol{x}^-)) < 1 - \frac{(1-\sqrt{2n\sqrt{\epsilon}})^2}{\sigma_1^2};
    \end{align*}
    \item with probability $\frac{1}{n}$, 
    \begin{align*}
        g(\pi(\boldsymbol{x}))^\top g(\tau(\boldsymbol{x}^-)) \leq 1.
    \end{align*}
\end{enumerate}
\end{lemma}

\begin{proof}
\begin{enumerate}
    \item  With probability $\frac{n-1}{n}$, for $(\boldsymbol{x},y), (\boldsymbol{x}^-, y^-)\sim \Dd$, $y\neq y^-$. 
    By Lemma \ref{lem:markov-inequ},
    with probability at least $1-2\sqrt{\epsilon}$,
    \begin{align*}
        ||h\circ g(\pi(\boldsymbol{x})) - \boldsymbol{e}_y|| < \sqrt{n \sqrt{\epsilon}} 
        \text{\ \ \ \ and\ \ \ \ }
        ||h\circ g(\tau(\boldsymbol{x}^-)) - \boldsymbol{e}_{y^-}|| < \sqrt{n \sqrt{\epsilon}}.
    \end{align*}
    By the triangle inequality,
    \begin{align*}
        ||g(\pi(\boldsymbol{x})) - g(\tau(\boldsymbol{x}^-))|| 
        &\geq \frac{1}{\sigma_1} ||h\circ g(\pi(\boldsymbol{x})) - h\circ g(\tau(\boldsymbol{x}^-))||\\
        &\geq \frac{1}{\sigma_1} (||\boldsymbol{e}_y - \boldsymbol{e}_{y^-}|| - ||h\circ g(\pi(\boldsymbol{x})) - \boldsymbol{e}_y|| - ||h\circ g(\tau(\boldsymbol{x}^-)) - \boldsymbol{e}_{y^-}||)\\
        &> \frac{\sqrt{2} - 2\sqrt{n\sqrt{\epsilon}} } {\sigma_1}.
    \end{align*}
    Since $g$ is normalized,
    \begin{align*}
        g(\pi(\boldsymbol{x}))^\top g(\tau(\boldsymbol{x}^-)) &= 1 - \frac{1}{2}||g(\pi(\boldsymbol{x}))- g(\tau(\boldsymbol{x}^-))||^2 \\ 
        &< 1 - \frac{(1-\sqrt{2n\sqrt{\epsilon}})^2}{\sigma_1^2}.
    \end{align*}

    \item As we assume $\Dd$ is a balanced dataset, with probability $\frac{1}{n}$, for $(\boldsymbol{x},y), (\boldsymbol{x}^-, y^-)\sim \Dd$, $y= y^-$. 
    Since $g$ is normalized,
    \begin{align*}
        g(\pi(\boldsymbol{x}))^\top g(\tau(\boldsymbol{x}^-)) &= 1 - \frac{1}{2}||g(\pi(\boldsymbol{x}))- g(\tau(\boldsymbol{x}^-))||^2 \\ 
        &\leq  1 - \frac{1}{2\sigma_1^2}||h\circ g(\pi(\boldsymbol{x}))- h\circ g(\tau(\boldsymbol{x}^-))||^2 \\
        &\leq 1.
    \end{align*}

\end{enumerate}
\end{proof}

\subsection{Proof of Proposition \ref{thm:main}}
\begin{proof}
    Let $\Ee_{\SL} = \epsilon$.
    Combining Lemma \ref{lem:same-x} and Lemma \ref{lem:diff-x}, for a sample $\boldsymbol{x}$ and its negative sample $\boldsymbol{x}^-$ i.i.d from  $\Dd_{\boldsymbol{x}}$, and data augmentation method $\pi$, $\tau$, $\rho$ i.i.d from  $\mu$, with probability at least $1-4\sqrt{\Ee_{\SL}}$, it holds that
    \begin{align*}
    \L_{\CL}(x, x^-, \pi, \tau, \rho)
    =&-\log \frac{e^{g(\pi(\boldsymbol{x}))^\top g(\tau(\boldsymbol{x}))}}{e^{g(\pi(\boldsymbol{x}))^\top g(\tau(\boldsymbol{x}))}+e^{g(\pi(\boldsymbol{x}))^\top g(\rho(\boldsymbol{x}^-))}}\\
    =& \log (1 + \frac{e^{g(\pi(\boldsymbol{x}))^\top g(\rho(\boldsymbol{x}^-))}}{e^{g(\pi(\boldsymbol{x}))^\top g(\tau(\boldsymbol{x}))}})\\
    <& \frac{n-1}{n} \log (1+\frac{1 - \frac{(1-\sqrt{2n\sqrt{\Ee_{\SL}}})^2}{\sigma_1^2}}{1 - \frac{2n\sqrt{\Ee_{\SL}}}{\sigma_n}}) 
    + \frac{1}{n} \log (1+\frac{1}{1 - \frac{2n\sqrt{\Ee_{\SL}}}{\sigma_n}})\\
    =&\frac{1}{n} \log (1+\frac{\sigma_n}{\sigma_n-2n\sqrt{\Ee_{\SL}}})+ \frac{n-1}{n} \log (1+\frac{\sigma_1^2\sigma_n-\sigma_n(1-\sqrt{2n\sqrt{\Ee_{\SL}}})^2}{\sigma_1^2\sigma_n-2n\sigma_1^2\sqrt{\Ee_{\SL}}}).
    \end{align*}
\end{proof}

\end{document}